\documentclass{lmcs}

\usepackage[ruled,noend,linesnumbered]{algorithm2e}
\SetKwInput{KwInit}{Init}

\usepackage{tikz}
\usetikzlibrary{automata,arrows,arrows.meta,calc,decorations.pathreplacing,plotmarks,shapes,tikzmark,matrix,fit,positioning}
\tikzset{comp state/.style={draw,rectangle,rounded corners,inner sep=2pt}}

\usepackage{pgfplots}
\usepackage{pgfplotstable}
\usepackage{booktabs}
\usepackage{xspace}
\usepackage{stmaryrd}
\usepackage{xcolor}
\usepackage{hyperref}
\usepackage{amssymb}
\usepackage{caption}
\usepackage{subcaption}
\usepackage{cleveref}

\Crefname{exa}{Example}{Examples}
\Crefname{defi}{Definition}{Definitions}
\Crefname{thm}{Theorem}{Theorems}
\Crefname{lem}{Lemma}{Lemmata}
\Crefname{cor}{Corollary}{Corollaries}
\Crefname{rem}{Remark}{Remarks}
\Crefname{prop}{Proposition}{Propositions}

\newcommand{\Lstar}{\ensuremath{\mathtt{L^*}}\xspace}
\newcommand{\Lsharp}{\ensuremath{\mathtt{L^{\#}}}\xspace}

\newcommand{\Lstarqb}{\ensuremath{\mathtt{L^*_{?,b}}}\xspace}

\newcommand{\qm}{\ensuremath{\mathord{?}}\xspace}
\newcommand{\row}{\mathit{row}}
\newcommand{\weakeq}{\approx}
\newcommand{\pref}{\mathit{Pref}}
\newcommand{\suff}{\mathit{Suff}}
\newcommand{\Equiv}{\mathit{Equiv}}
\newcommand{\Member}{\mathit{Member}}

\newcommand{\eg}{\emph{e.g.}\xspace}
\newcommand{\ie}{\emph{i.e.}\xspace}
\newcommand{\Obs}{\mathsf{Obs}}
\newcommand{\Dom}{\mathsf{Dom}}

\newcommand{\ab}{\Sigma}
\newcommand{\multiab}{\Sigma^m}
\newcommand{\projection}[2]{{#1}_{\upharpoonright #2}}
\newcommand{\partord}{\ensuremath{\preccurlyeq}}
\newcommand{\strpartord}{\ensuremath{\prec}}
\newcommand{\SUL}{\mathit{SUL}}

\newcommand{\edgerel}{\ensuremath{\to}}
\newcommand{\init}[1]{\ensuremath{\hat{#1}}\xspace}
\newcommand{\transition}[1]{\mathrel{\smash{\raisebox{-1pt}{$\xrightarrow{#1}$}}}}
\newcommand{\parcomp}{\mathop{\parallel}}
\newcommand{\bigpar}{\mathop{\smash{\scalebox{1.2}{$\parallel$}}}}
\newcommand{\coala}{\textsc{Coal}\textit{A}\xspace}
\newcommand{\coal}{\textsc{Coal}\xspace}
\newcommand{\lang}{\mathcal{L}}
\newcommand{\Set}{\mathcal{P}}
\newcommand{\lts}{T}
\newcommand{\symdif}{\mathrel{\Delta}}
\newcommand{\CH}{\mathcal{H}}
\newcommand{\pfun}{\rightharpoonup}

\newcommand{\Teach}{\mathit{Teacher}}
\newcommand{\Learn}{L}
\newcommand{\Learns}{\ensuremath{\mathbb{L}}}

\newcommand{\CED}{\mathit{CED}}
\newcommand{\CEX}{\mathit{CEX}}
\newcommand{\can}[1]{\llbracket #1 \rrbracket_{\partord}}
\newcommand{\Pos}{P}
\newcommand{\Neg}{\sigma_N}
\newcommand{\supp}{\mathsf{img}}
\newcommand{\nodiff}{\theta}
\newcommand{\DSet}{\mathcal{D}}
\newcommand{\Dis}{\delta}

\newcommand{\plus}{\mathord{+}}
\newcommand{\minus}{\mathord{-}}

\newcommand{\viz}{\emph{viz.}\xspace}

\begin{document}
\title{A Detailed Account of Compositional Automata Learning through Alphabet Refinement}

\titlecomment{{\lsuper*} The current paper is an extension of~\cite{NeeleS23,HenryMNS25}}

\author[L. Henry]{L\'eo Henry\lmcsorcid{0000-0001-6778-5840}}[a]
\author[M.R. Mousavi]{Mohammad Reza Mousavi\lmcsorcid{0000-0002-4869-6794}}[b]
\author[T. Neele]{Thomas Neele\lmcsorcid{0000-0001-6117-9129}}[c]
\author[M. Sammartino]{Matteo Sammartino\lmcsorcid{0000-0003-1456-2242}}[d]

\address{Aix-Marseille Université, Marseille, France}
\email{leo.henry@univ-amu.fr}
\address{King's College London, London, United Kingdom}
\email{mohammad.mousavi@kcl.ac.uk}
\address{Eindhoven University of Technology, Eindhoven, The Netherlands}
\email{t.s.neele@tue.nl}
\address{Royal Holloway, University of London, Egham, United Kingdom}
\email{Matteo.Sammartino@rhul.ac.uk}

\begin{abstract}
Active automata learning infers automaton models of systems from behavioral observations, a technique successfully applied to a wide range of domains.
Compositional approaches have recently emerged to address scalability to concurrent systems.
We take a significant step beyond available results, including those by the authors, and develop a general technique for compositional learning of a synchronizing parallel system with an unknown decomposition.
Our approach \emph{automatically refines} the global alphabet into component alphabets while learning the component models.
We develop a theoretical treatment of \emph{distributions} of alphabets, i.e., sets of possibly overlapping component alphabets, characterize \emph{counter-examples} that reveal inconsistencies with global observations, and show how to systematically update the distribution to restore consistency.
We extend $\Lstar$ to handle partial and potentially spurious information arising when learning components from global observations only.
We establish correctness and termination of the full algorithm.
We provide an implementation, called \coala, using the state-of-the-art active learning library \textsc{LearnLib}.
Our experiments on more than 630 subject systems show that \coala delivers up to five orders of magnitude fewer membership queries than monolithic learning, and achieves better scalability in equivalence queries on systems with significant concurrency.
\end{abstract}

\maketitle

\section{Introduction}
Automata learning~\cite{Gold78} has been successfully applied to learn widely-used protocols such as TCP~\cite{FiterauBrosteanJV16}, SSH~\cite{FiterauBrosteanLPR17}, and QUIC~\cite{FerreiraBDS21},  CPU caching policies~\cite{VilaGGK20} and finding faults in their black-box actual implementations. Moreover it has been applied to a wide range of applications such as bank cards \cite{AartsRP13} and biometric passports \cite{AartsSV10,MarksteinerEtAl2024}. There are already accessible expositions of the success stories in this field~\cite{HowarS16,Vaandrager17}. However, it is well known that the state-of-the-art automata learning algorithms do not scale beyond systems with more than a few hundreds of input and output symbols in their alphabets~\cite{Vaandrager17,fortz_research_2026}.

Scalability to larger systems with large input alphabets is required for many real-life systems. This requirement has inspired some recent attempts~\cite{LabbafGHM23,Samadi2025} to come up with compositional approaches to automata learning, to address the scalability issues.  Indeed, if $\Lstar$ is applied to the composite system as a whole, the number of queries is proportional to the full state space, which grows exponentially in the number of components. In practice, queries are implemented as tests on the system, and when tests are costly, for instance when they must be repeated for accuracy or involve physical interaction, learning the whole system may be impractical. Some of the past compositional approaches~\cite{LabbafGHM23} considered non-synchronizing state machines or a subset of synchronizing ones to help learning the decomposition~\cite{Samadi2025}.
Our own earlier work assumed the decomposition to be known a priori~\cite{NeeleS23}. These assumptions need to be relaxed to enable dealing with legacy and black box systems and dealing with synchronizing components. In particular, in the presence of an ever-increasing body of legacy automatically-generated code, architectural discovery is a significant challenge~\cite{Koschke2009,FuchssLiuHey2025,Marc2024}. This is especially pertinent given that automata learning has already been successfully applied to the refactoring and rejuvenation of legacy systems~\cite{SchutsHV16}, making architectural discovery a prerequisite for scaling such applications. 

In this paper, we take a significant step beyond the available results, developing a compositional automata learning approach that assumes no prior knowledge of the decomposition of the alphabet and allows for an arbitrary general synchronization scheme, common in the theory of automata and process calculi~\cite{Hoare2004}. 
To this end, we take inspiration from the realizability problem in concurrency theory and 
use iterative refinements of the alphabet decomposition (called distributions~\cite{Mukund2002}) to arrive at a provably sound decomposition while learning the components' behavior. The algorithm orchestrates multiple instances of an extended $\Lstar$, coordinated by an intermediate layer that translates between local component queries and global system queries. Several challenges arise from the synchronizing nature of the components. The answer to some membership queries for a specific component may be \emph{unknown} if the correct sequence of interactions with other components has not yet been observed, and counter-examples for the global system may yield \emph{spurious} local counter-examples that need to be corrected later. Furthermore, the decomposition itself must be discovered: when a global counter-example reveals an inconsistency in the current distribution, it must be resolved and the local learners restarted, in a process that must be guaranteed to converge. We address all of these challenges in a unified framework.

To our knowledge this is the first result of its kind and the first extension of realizability into the domain of automata learning. 
This paper extends our earlier conference publications \cite{HenryMNS25,NeeleS23} by including full proofs of the results and integrating and contextualizing the contributions of both conference papers. Moreover, we expand on the presentation and analysis of our empirical evaluation. 

To summarize, the contributions of our paper are listed below: 
\begin{itemize}
    \item We develop a novel theory of system decomposition for LTS synchronization that formally characterizes which alphabet decompositions can accurately model observed behaviors, establishing a theoretical foundation for automated component discovery.
    \item We design a query translation layer that mediates between local and global queries, and extend $\Lstar$ to handle the resulting incomplete information and spurious counter-examples, via wildcards and a backtracking mechanism.
    \item Based on the above, we propose a compositional active learning algorithm that dynamically refines component alphabets during the learning process, supporting standard synchronization mechanisms without requiring a priori knowledge of the system's decomposition.
    \item 
    We implemented our approach as the prototype tool \coala, built on the state-of-the-art LearnLib framework~\cite{IsbernerHS15}, and evaluated it on over 630 systems from three benchmark sets. Compared to a monolithic approach, \coala achieved substantial reductions in queries, with up to five orders of magnitude fewer membership queries and one order fewer equivalence queries across most of our benchmark systems with parallel components, resulting in better overall scalability. The replication package is available at~\cite{HenryMNS25Rep}.
\end{itemize}

\section{Related Work}
\label{sec:related_work}
Finding ways of projecting a known concurrent system down into its components is the subject of several works, e.g., \cite{CastellaniMT99,GrooteM92,Mukund2002}.
In principle, it would be possible to learn the system monolithically and use the aforementioned results.
However, as demonstrated in our earlier work~\cite{NeeleS23}, this may result in a substantial query blow-up.
For example, the approach of Bollig et al.~\cite{BolligKKL10} learns asynchronously-communicating finite state machines via queries in the form of message sequence charts.
The result is a monolithic DFA that is later broken down into components via an additional synthesis procedure.
Such an approach does not mitigate the exponential blow up in the number of queries due to the interleaving in the semantics of the parallel model. 
Another example is the approach by van Heerdt et al.~\cite{HeerdtKR021}, which provides an extension of \Lstar to \emph{pomset automata}.
These automata are acceptors of partially-ordered multisets, which model concurrent computations.
This  approach relies on an oracle capable of processing pomset-shaped queries; devising such an oracle will require substantial work, such as the approach reported in our paper, and a naive attempt to an ordinary sequential oracle may cause a query blow-up.

Our work integrates two recent approaches on compositional learning: we extend the work on learning synchronous parallel composition of automata~\cite{NeeleS23} by automatically learning the decomposition of the alphabets, through refinement of distributions; moreover, we extend the work on learning interleaving parallel composition of automata~\cite{LabbafGHM23} by enabling a generic synchronization scheme among components.
In parallel to our work, an alternative proposal~\cite{Samadi2025} has appeared to allow for synchronization among interleaving automata; however, the proposed synchronization scheme assumes that whenever two components are not ready to synchronize, e.g., because they produce different outputs on the same input, a special output is produced (or otherwise, the semantic model is output deterministic).
We do not assume any such additional information and use a synchronization scheme widely used in the theory of automata and process calculi~\cite{Hoare2004}.

Other contributions related to compositional learning include 
the active learning of \emph{product automata}, a variation of Mealy machine where the output is the combination of outputs of several Mealy machines -- as in our case, the component Mealy machines are learned individually~\cite{Moerman18}; learning of \emph{systems of procedural automata}~\cite{FrohmeS21}, sets of automata that can call each other in a way similar to procedure calls; learning asynchronously-communicating finite state machines via queries in the form of message sequence charts~\cite{BolligKKL10}, though using a monolithic approach.

Other approaches consider Teachers that are unable to reply to membership queries~\cite{AbelR16,GrinchteinL06,GrinchteinLP06,LeuckerN12}; they all use SAT-based techniques to construct automata.
The closest works to ours are by Grinchtein et al.~\cite{GrinchteinLP06}, considering the problem of compositionally learning a property of a concurrent system with full knowledge of the components; and in the approach by Abel and Reineke \cite{AbelR16}, learning an unknown component of the \emph{serial} composition of two automata.
In none of these works spurious counter-examples arise.

Realizability of sequential specifications in terms of parallel components has been a long-standing problem in concurrency theory. 
In the context of Petri nets, this has been pioneered by Ehrenfeucht and Rozenberg \cite{EhrenfeuchtR89a}, followed up by the work of Castellani, Mukund and Thiagarajan \cite{CastellaniMT99}.
Realizability has been further investigated in other models of concurrency such as team automata \cite{BeekHP24}, session types \cite{Barbanera2021}, communicating automata \cite{Guanciale2018} and labeled transition systems (LTSs)~\cite{Mukund2002}.
Related to this line of research is the decomposition of LTSs into prime processes \cite{Lut16}. 
We are inspired by the work of Mukund~\cite{Mukund2002}, characterizing the transition systems that can be synthesized into an equivalent parallel system given a decomposition of their alphabet (called distribution).
Mukund explores this characterization for two notions of parallel composition (loosely cooperating- and synchronous parallel composition) and three notions of equivalence (isomorphism, language equivalence, and bisimulation).
We base our work on the results of Mukund for loosely cooperating systems and language equivalence. We extend it to define consistency between observations and distributions and refining distributions to reinstate consistency.

\section{Preliminaries}
\label{sec:preliminaries}

\paragraph{Notation and terminology.}
We use $\ab$ to denote a \emph{finite alphabet} of action symbols, and $\ab^\star$ to denote the set of finite sequences of symbols in $\ab$, which we call \emph{traces}; we use $\epsilon$ to denote the empty trace.
Given two traces $s_1,s_2 \in \ab^\star$, we denote their concatenation by $s_1 \cdot s_2$; for two sets $S_1,S_2 \subseteq \ab^\star$, $S_1 \cdot S_2$ denotes element-wise concatenation.
Given $s \in \ab^\star$, we denote by $\pref(s)$ the set of prefixes of $s$, and by $\suff(s)$ the set of its suffixes; the notation lifts to sets $S \subseteq \ab^\star$ as expected.
We say that $S \subseteq \ab^\star$ is \emph{prefix-closed} (resp.\ \emph{suffix-closed}) whenever $S = \pref(S)$ (resp.\ $S = \suff(S)$).
The \emph{projection} $\projection{\sigma}{\ab'}$ of $\sigma$ on an alphabet $\ab' \subseteq \ab$ is the sequence of symbols in $\sigma$ that are also contained in $\ab'$ : \(\projection{\epsilon}{\ab'}=\epsilon\) and \(\projection{\sigma\cdot a}{\ab'}=\projection{\sigma}{\ab'}\cdot a\) if \(a\in\ab'\) and \(\projection{\sigma}{\ab'}\) otherwise. 
We generalize this notation to sets (and thus languages), such that 
\(\projection{S}{\ab'}=\{\projection{\sigma}{\ab'}\mid\sigma\in S\}\).
Finally, given a set $S$, we write $|S|$ for its cardinality.

\subsection{Labeled Transition Systems and their Languages}

In this work we represent the state-based behavior of a system as a \emph{labeled transition system}.

\begin{defi}[Labeled Transition System]
	\label{def:lts}
	A \emph{labeled transition system} (LTS) is a four-tuple $L = (S,\ab,\mathord\edgerel,\linebreak[1]\init{s})$, where
	\begin{itemize}
		\item $S$ is a set of states, which we refer to as the \emph{state space};
		\item $\ab$ is a finite set of actions, called the \emph{alphabet};
		\item $\edgerel\, \subseteq S \times \ab \times S$ is a transition relation; and
		\item $\init{s} \in S$ is an initial state.
	\end{itemize}
	We say that $L$ is \emph{deterministic} whenever for each $s \in S$, $a \in \ab$ there is at most one transition from $s$ labeled by $a$.
\end{defi}

We write in infix notation $s \transition{a} t$ for $(s,a,t) \in \mathnormal{\edgerel}$. 
We say that an action $a$ is
\emph{enabled} in $s$, written $s \transition{a}$, if there is $t$ such that $s \transition{a} t$.
The transition relation and the notion of enabled-ness are also extended to traces $\sigma \in \ab^\star$, yielding $s \transition{\sigma} t$ and $s \transition{\sigma}$.
\begin{defi}[Language of an LTS]
    The language of $\lts$ is the set of traces enabled from the starting state, formally:
\[
	\lang(\lts) = \{ \sigma \in \ab^\star \mid \init{s} \transition{\sigma} \} \enspace .
\]
\label{def:lts-lang}
\end{defi}

From here on, we only consider deterministic LTSs.
Note that this does not reduce the expressivity, in terms of the languages that can be encoded.
\begin{rem}
\label{rem:lts-dfa}
Languages of LTSs are always prefix-closed, because every prefix of an enabled trace is necessarily enabled. Prefix-closed languages are accepted by a special class of deterministic finite automata (DFA), where all states are final except for a sink state, from which all transitions are self-loops.
Our implementation (see Section~\ref{sec:exp}) uses these models as underlying representation of LTSs.
\end{rem}

The parallel composition of a finite set of LTSs is a product model representing all possible behaviors when the LTSs synchronize on shared actions.
Intuitively, an action $a$ can be performed when all LTSs that have \(a\) in their alphabet can perform it in their current state. The other LTSs remain idle during the transition.

\begin{defi}[Parallel composition]
	\label{def:n_parallel_composition}
	Given $n$ LTSs $\lts_i = (S_i, \ab_i, \edgerel_i, \init{s}_i)$ for $1 \leq i \leq n$, their \emph{parallel composition}, denoted $\bigpar_{i=1}^n \lts_i$, is an LTS $(S_1 \times \dots \times S_n, \bigcup_{i=1}^n \ab_i, \transition{}, (\init{s}_1, \dots, \init{s}_n))$, where the transition relation $\transition{}$ is given by the following rule:
	\[
		\frac
		{%
		\begin{alignedat}{2}
			s_i \transition{a}_i t_i \quad & \text{for all $i$ such that $a \in \ab_i$} \\
			s_j = t_j \quad & \text{for all $j$ such that $a \notin \ab_j$}
		\end{alignedat}
		}
		{(s_1,\dots,s_n) \transition{a} (t_1,\dots,t_n)}
	\]
\end{defi}

We say that an action $a$ is \emph{local} if there is exactly one $i$ such that $a \in \ab_i$, otherwise it is called \emph{synchronizing}.
The parallel composition of LTSs thus forces individual LTSs to cooperate on synchronizing actions; local actions can be performed independently.
We typically refer to the LTSs that make up a composite LTS as \emph{components}.
Synchronization of components corresponds to communication between components in real-world settings.

As clear from the definition, the paths of a parallel composition are exactly the composition of paths for each component. 
\begin{lem}
    \label{lm:comp_int_aut}
    Given $n$ LTSs $\lts_i = (S_i, \ab_i, \edgerel_i, \init{s}_i)$ for $1 \leq i \leq n$ and their parallel composition \(\bigpar_{i=1}^n \lts_i\), for any word \(\sigma\in(\bigcup_{i=1}^n \ab_i)^{*}\), 
    \((\init{s}_1, \dots, \init{s}_n)\transition{\sigma}(t_1,\dots,t_n)
    \Leftrightarrow \forall i\, \init{s}_i\transition{\projection{\sigma}{\ab_i}}t_i\)
\end{lem}
\begin{proof}
    By induction on the length of \(\sigma\).
\end{proof}

For the purposes of active learning, it is often useful to reason on languages instead of automata. We hence define the notion of parallel composition for languages on restricted alphabets, following the intuition detailed in \cref{lm:comp_int_aut}.
\begin{defi}[Parallel composition of languages]
    Given $n$ languages and alphabets \((\lang_i,\ab_i)\) such that \(\lang_i\subseteq\ab_i^\star\) for all $1 \leq i \leq n$, let $\ab = \bigcup_{i=1}^{n} \ab_i$. We define $\bigpar_{i=1}^{n} (\lang_i,\ab_i)$ as
    \[ \parcomp_{i=1}^{n} (\lang_i,\ab_i) = \{ \sigma\in\ab^\star \mid \forall 1 \leq i \leq n \ldotp \projection{\sigma}{\ab_i} \in \lang_i \}\ .\]
\end{defi}
\begin{lem}
	Given $n$ LTSs $\lts_i = (S_i, \ab_i, \edgerel_i, \init{s}_i)$ for $1 \leq i \leq n$, 
	\[\parcomp_{i=1}^{n} (\lang(\lts_i),\ab_i) = \lang(\parcomp_{i=1}^n \lts_i)\]
\end{lem}
\begin{proof}
   For any \(\lts_i\) we have \(\lang(\lts_i)=\{\sigma_i \in \ab_i^\star \mid \init{s}_i \transition{\sigma_i}_i\}\), hence 
   \[ 
   \parcomp_{i=1}^{n} (\lang(\lts_i),\ab_i)=\{\sigma \in \ab^\star \mid \forall 1\leq i \leq n\ldotp \init{s}_i\transition{\projection{\sigma}{\ab_i}}_i \}
   \]
   By \cref{lm:comp_int_aut}, this set of traces is exactly \(\lang(\parcomp_{i=1}^n \lts_i)\).
\end{proof}

\begin{exa}[Running example]
\label{ex:running}
	\label{ex:parallel_composition}
	Consider the LTSs $\lts_1$ and $\lts_2$ given in \Cref{fig:running}, with the respective alphabets $\{a,b,c\}$ and $\{b,d\}$. Their parallel composition is depicted at the bottom of \Cref{fig:running}.
    \begin{figure}[t]
	\centering
    \begin{tikzpicture}[->,>=stealth',auto,semithick, initial text={}, initial distance={10pt}, every state/.style={inner sep=1pt, minimum size=0pt},clip=true]
        \def\d{1.4}

        \begin{scope}[xshift=-3cm]
            \node                  at (-1.2,0) {$\lts_1 = $};
            \node[state,initial] (s0) at (0,0)    {$s_0$};
            \node[state]      (s1) at (\d,0)   {$s_1$};
            \node[state]      (s2) at (2*\d,0) {$s_2$};

            \path
            (s0) edge[bend left] node  {$a$} (s1)
            (s0) edge[bend right] node {$c$} (s1)
            (s1) edge[loop above] node {$b$} (s1)
            (s1) edge node {$c$} (s2);
        \end{scope}

        \begin{scope}[xshift=3cm]
            \node                  at (-1.2,0) {$\lts_2 = $};
            \node[state,initial] (t0) at (0,0)    {$t_0$};
            \node[state]      (t1) at (\d,0)   {$t_1$};
            \node[state]      (t2) at (2*\d,0) {$t_2$};

            \path
            (t0) edge node {$b$} (t1)
            (t1) edge node {$d$} (t2);
        \end{scope}

        \begin{scope}[xshift=-2cm, yshift=-1.6cm]
            \def\x{2}
            \def\y{1}
            \node at (-1.8,0)   {$ \lts_1 \parcomp \lts_2 = $};
            
            \node[comp state,initial] (11) at (0,0)   {$s_0,t_0$};
            \node[comp state] (21) at (\x,0) {$s_1,t_0$};      
            \node[comp state] (31) at (\x,-\y) {$s_2,t_0$};
            \node[comp state] (22) at (2*\x,0) {$s_1,t_1$};
            \node[comp state] (32) at (2.8*\x,0.7*\y) {$s_2,t_1$};
            \node[comp state] (33) at (3.6*\x,0) {$s_2,t_2$};
            \node[comp state] (23) at (2.8*\x,-0.7*\y) {$s_1,t_2$};				

            \path
            (11) edge[bend left=20]    node {$a$} (21)
            (11) edge[bend right=20]   node {$c$} (21)
            (21) edge[']               node {$c$} (31)
            (21) edge                  node {$b$} (22)
            (22) edge[']               node {$d$} (23)
            (22) edge                  node {$c$} (32)
            (32) edge				   node {$d$} (33)
            (23) edge[']               node {$c$} (33);
        \end{scope}
    \end{tikzpicture}
    \caption{The parallel composition of two LTSs.}
    \label{fig:running}
    \end{figure}
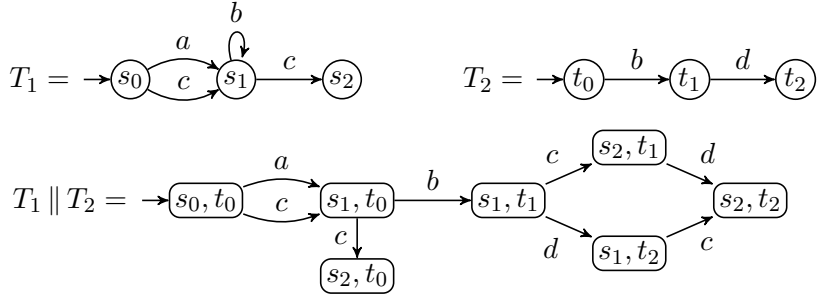
	Here $a$, $c$ and $d$ are local actions, whereas $b$ is synchronizing. Note that, although $\lts_2$ can perform $b$ from its initial state $t_0$, there is no $b$ transition from $(s_0,t_0)$ in $\lts_1 \parcomp \lts_2$, because $b$ is not enabled in $s_0$. Action $b$ can only be performed in $\lts_1 \parcomp \lts_2$ after $\lts_1$ does an $a$ or a $c$ and moves to $s_1$, which is captured as the $a$ and $c$ transitions from $(s_0,t_0)$.
\end{exa}

\subsection{Active Automata Learning}
In active automata learning~\cite{Angluin87}, a \emph{Learner} infers an automaton model of an unknown language $\lang$ by querying a \emph{Teacher}, which knows $\lang$ and answers two query types (see~\Cref{fig:mat}):
\begin{itemize}
    \item \emph{Membership queries}: is a trace $\sigma$ in $\lang$? The Teacher replies yes/no.
    \item \emph{Equivalence queries}: given a \emph{hypothesis} model $\CH$, is $\lang(\CH) = \lang$? The Teacher either replies yes or provides a \emph{counter-example} -- a trace that is in one language but not in the other.
\end{itemize}
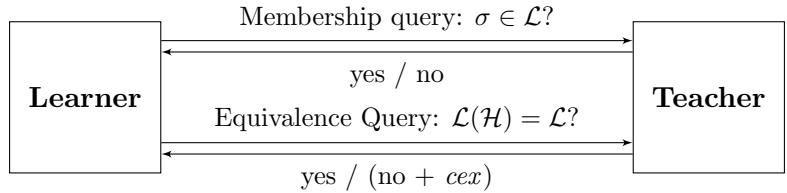
\begin{figure}[ht]
    \centering
    \begin{tikzpicture}[scale=.75]
        \node[draw, minimum width=2cm, minimum height=2cm] (learner) at (0,0) {\bf Learner};
        \node[draw, minimum width=2cm, minimum height=2cm] (Teacher) at (11,0) {\bf Teacher};

\def\topA{1cm}
        \def\topB{0.8cm}
        \def\botA{-0.8cm}
        \def\botB{-1cm}

\path[-{latex'}] 
        ([yshift=\topA]learner.east) edge node[above] {\small Membership query: $\sigma \in \lang?$} ([yshift=\topA]Teacher.west)
        ([yshift=\topB]Teacher.west) edge node[below] {\small yes / no} ([yshift=\topB]learner.east)
        ([yshift=\botA]learner.east) edge node[above] {\small Equivalence Query: $\lang(\CH) = \lang?$} ([yshift=\botA]Teacher.west)
        ([yshift=\botB]Teacher.west) edge node[below] {\small yes / (no + $\mathit{cex})$} ([yshift=\botB]learner.east);
   
    \end{tikzpicture}
    \caption{Active automata learning for a target language $\lang$.}\vspace*{-.4cm}
    \label{fig:mat}
    \end{figure}
Algorithms based on this framework converge to a canonical model (\eg, the minimal DFA) of the target language.
In practice when learning from a \emph{black box} System Under Learning (SUL), the Teacher is realized as an interface to the SUL: membership queries become tests on the SUL, and equivalence queries are approximated via systematic testing strategies \cite{BroyEtAl2004,KrugerJR24}.

During learning, the learner gathers \emph{observations} about the SUL. 
While these observations are typically organized in a data structure (e.g., a table or a tree), they can be abstractly represented as a partial function mapping traces to their accepted ($\plus$) or rejected ($\minus$) status.

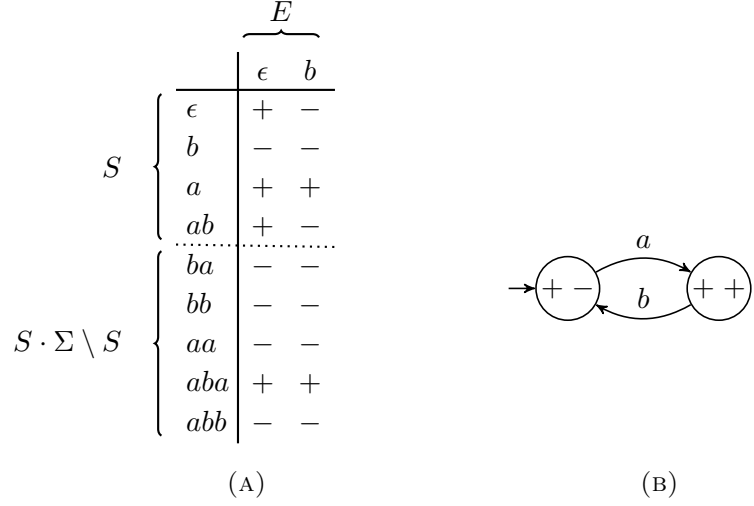
\begin{defi}[Observation function]
    An \emph{observation function} over \(\ab\) is a partial function 
    \(\Obs:\ab^\star\pfun\{\plus,\minus\}\). 
\end{defi}

We write \(\Dom(\Obs)\) for the domain of \(\Obs\) and only consider observation functions with a finite domain.
We sometimes represent an observation function \(\Obs\) as the set of pairs \(\{(\sigma,\Obs(\sigma))\mid\sigma\in\Dom(\Obs)\}\). 

\begin{defi}[Observation function/language agreement]
    An observation function \(\Obs\) agrees with a language \(\lang\), notation \(\lang\models\Obs\), 
    whenever \(\sigma\in\lang \Leftrightarrow \Obs(\sigma)=+\), for all \(\sigma\in\Dom(\Obs)\).
\end{defi}
\subsection{\Lstar algorithm}
The classical and most widely used active automata learning algorithm is \Lstar. Although the algorithm targets DFAs, we will present it in terms of deterministic LTSs, as they are our target model in this paper.

In \Lstar, the learner organises observations in a \emph{observation table}, which is a triple $(S,E,T)$, consisting of a finite, prefix-closed set $S \subseteq \ab^\star$, a finite, suffix-closed set $E \subseteq \ab^\star$, and a function $T : (S \cup S \cdot \ab)\cdot E \to \{+,-\}$.
The function $T$ can be seen as a table in which rows are labeled by traces in $S \cup S \cdot \Sigma$, columns by traces in $E$, and cells $T(s \cdot e)$ are set to $\Obs(s \cdot e)$.\footnote{Note that the traditional \Lstar assumes $\lang \models \Obs$; in our setting we sometimes deviate from this assumption.}

\begin{exa}
	\label{ex:obs-table}
	Consider the prefix-closed language $\lang$ over the alphabet $\ab = \{a,b\}$ consisting of traces where $a$ and $b$ alternate, starting with $a$; for instance $aba \in \lang$ but $abb \notin \lang$.
	An observation table generated by a run of \Lstar targeting this language is shown in Figure~\ref{fig:obs_table}.
	\qed
\end{exa}

\begin{figure}[t]
	\centering
	\begin{subfigure}{0.4\textwidth}
		\begin{tikzpicture}
			\tikzset{
				table/.style={
					matrix of math nodes,
					column sep=-\pgflinewidth,
					nodes={rectangle,text width=1em,text height=.6em,align=center},
					nodes in empty cells,
					ampersand replacement=\&
				},
				column 1/.style={nodes={align=left}}
			}
			\matrix (T) [table] {
				\&[.5em] \epsilon \& b \\
				\epsilon \& \plus \& \minus \\
				b \& \minus \& \minus \\
				a \& \plus \& \plus \\
				ab \& \plus \& \minus \\
				ba \& \minus \& \minus \\
				bb \& \minus \& \minus \\
				aa \& \minus \& \minus \\
				aba \& \plus \& \plus \\
				abb \& \minus \& \minus \\
			};
			\draw[decorate,decoration={brace,mirror},transform canvas={xshift=-.5em},thick] ([yshift=-2pt]T-2-1.north west) -- node[left=1em] {$S$} ([yshift=2pt]T-5-1.south west);
			\draw[decorate,decoration={brace,mirror},transform canvas={xshift=-.5em},thick] ([yshift=-2pt]T-6-1.north west) -- node[left=1em] {$S \cdot \ab \setminus S$} ([yshift=2pt]T-10-1.south west);
			\draw[decorate,decoration={brace},transform canvas={yshift=.5em},thick] ([xshift=2pt]T-1-2.north west) -- node[above=2pt] {$E$} ([xshift=-6pt]T-1-3.north east);
			\draw [thick] (T-1-2.north west) -- (T-10-2.south west);
			\draw [thick] (T-1-1.south west) -- (T-1-3.south east);
			\draw [thick,dotted] (T-5-1.south west) -- (T-5-3.south east);
		\end{tikzpicture}
		\caption{}
		\label{fig:obs_table}
	\end{subfigure}
	\begin{subfigure}{0.3\textwidth}
		\def\v{1.7cm}
		\vspace*{\v}
		\begin{tikzpicture}[->,>=stealth',auto,semithick, initial text={}, initial distance={10pt}, every state/.style={inner sep=1pt, minimum size=0pt}]
			\tikzstyle{state} = [draw,circle,inner sep=1pt];
			\tikzstyle{init} = [pin={[pin edge={black,<-},pin distance=3mm]left:}];

			\node[state,initial] (s0) at (0,0)   {$\plus \; \minus$};
			\node[state]      (s1) at (2,0)   {$\plus \; \plus$};

			\path
			(s0) edge[bend left] node[above] {$a$} (s1)
			(s1) edge[bend left] node[above] {$b$} (s0);
		\end{tikzpicture}
		\vspace*{\v}
		\caption{}
		\label{fig:constructed_lts}
	\end{subfigure}
	\caption{A closed and consistent observation table and the LTS that can be constructed from it.}
	\label{fig:obs_table_and_lts}
\end{figure}

Let $row_T \colon S \cup S \cdot \ab \to (E \to \{\plus,\minus\})$ denote the function $row_T(s)(e) = T(s \cdot e)$ mapping each row of $T$ to its content (we omit the subscript $T$ when clear from the context). The crucial observation is that $T$ approximates the Nerode congruence~\cite{Nerode58} for $\lang$ as follows: $s_1$ and $s_2$ are in the same congruence class only if $row(s_1) = row(s_2)$, for $s_1,s_2 \in S$. Based on this fact, the learner can construct a hypothesis LTS from the table, in the same way the minimal DFA accepting a given language is built via its Nerode congruence:\footnote{For the minimal DFA, the set of states is $\{ row(s) \mid  s \in S\}$; here we only take accepting states as we are building an LTS.}
\begin{itemize}
	\item the set of states is $\{ row(s) \mid s \in S, row(s)(\epsilon) = \plus\}$ (because we target LTSs);
	\item the initial state is $row(\epsilon)$;
	\item the transition relation is given by $row(s) \transition{a} row(s \cdot a)$, for all $s \in S$ and $a \in \ab$.
\end{itemize}
In order for the transition relation to be well-defined, the table has to satisfy the following conditions:
\begin{itemize}
	\item \textbf{Closedness:} for all $s \in S, a \in \Sigma$, there is $s' \in S$ such that $row_T(s') = row_T(s \cdot a)$.
	\item \textbf{Consistency:} for all $s_1,s_2 \in S$ such that $row_T(s_1) = row_T(s_2)$, we have $\row_T(s_1 \cdot a) = row_T(s_2 \cdot a)$, for all $a \in \Sigma$.
\end{itemize}
\begin{exa}
	The table of Example~\ref{ex:obs-table} is closed and consistent. The corresponding hypothesis LTS, which is also the minimal LTS that has $\lang$ as its language, is shown in Figure~\ref{fig:constructed_lts}.
	\qed
\end{exa}
The algorithm works in an iterative fashion: starting from the empty table, where $S$ and $E$ only contain $\epsilon$, the learner extends the table via membership queries until it is closed and consistent, at which point it builds a hypothesis and submits it to the Teacher in an equivalence query. If a counter-example is received, it is incorporated in the observation table by adding its prefixes to $S$, and the updated table is again checked for closedness and consistency. The algorithm is guaranteed to eventually produce a hypothesis $H$ such that $\lang(H) = \lang$, for which an equivalence query will be answered positively, causing the algorithm to terminate.
\section{Foundation and Overview of the Approach}
\label{sec:overview}

The active learning setting described so far is monolithic: a single learner infers a model of the entire system. In this section, we present an overview of our approach to lift this setting to the compositional one. We first develop the necessary foundations (\Cref{subsec:distributions}): we define local observation functions derived from global ones, and characterize when a decomposition of the global alphabet (a \emph{distribution}) adequately models the global observations. We then present an overview of the algorithm that builds on these foundations (\Cref{subsec:overview}).

\subsection{Observations and Distributions.}
\label{subsec:distributions}
In a parallel composition $\parcomp_i \lts_i$, permuting symbols belonging to different local alphabets does not affect language membership. In \Cref{fig:running}, for instance, $abdc$ is in the language $\lang(\lts_1 \parcomp \lts_2)$ whenever $abcd$ is, since $c$ and $d$ belong to different components. A \emph{distribution} formalises this independence assumption by specifying which symbols belong to the same component. 
\begin{defi}[Distribution]
    A \emph{distribution} of an alphabet \(\ab\) is a set \(\Omega=\{\ab_1,\dots,\ab_n\}\) such that \(\bigcup_{i=1}^n \ab_i=\ab\). 
\end{defi}
For the rest of this subsection, we fix an alphabet \(\ab\), a distribution \(\Omega=\{
\ab_1,\dots,\ab_n\}\) of \(\ab\) and an observation function \(\Obs\) over \(\ab\) unless otherwise specified.

We then define how to derive a local observation function for each $\ab_i$ from the global observations, mimicking parallel composition: when \(\sigma\) is enabled in $\bigpar_{i=1}^n \lts_i$ then it must be enabled in all \(\lts_i\). 

\begin{defi}[Local observation function]
    Given a sub-alphabet \(\ab_i\subseteq\ab\),
    a local observation function
    \(\Obs_{\ab_i}:\ab_i^\star\pfun\{\plus,\minus\}\) is defined such that 
    \(\Dom(\Obs_{\ab_i})=\projection{\Dom(\Obs)}{\ab_i}\) and 
    \(\Obs_{\ab_i}(\sigma')=\bigvee_{\{\sigma\mid\sigma\in\Dom(\Obs)\wedge \projection{\sigma}{\ab_i}=\sigma'\}}\Obs(\sigma)\), for all \(\sigma'\in\Dom(\Obs_{\ab_i})\).
\end{defi}

\begin{exa}
    Consider again the LTSs from \Cref{fig:running} and suppose we are given the following observation function for \(\lts_1 \parcomp \lts_2\):
    \(\Obs: a \mapsto \plus; aa \mapsto \minus; abd \mapsto \plus\).
    The local observation functions we obtain for \(T_1\) and \(T_2\) are, respectively:
    \[\Obs_{\{a,b,c\}}: a \mapsto \plus; aa \mapsto \minus; ab \mapsto \plus \qquad \Obs_{\{b,d\}}: \epsilon \mapsto \plus; bd \mapsto \plus.\]
    The observation \(\Obs(abd) = \plus\) requires both components to cooperate, hence \(\Obs_{\{a,b,c\}}(ab) = \plus\) and \(\Obs_{\{b,d\}}(bd) = \plus\).
    We derive \(\Obs_{\{b,d\}}(\epsilon) = \plus\) from \(\Obs(a) \lor \Obs(aa) = \plus\), since the projection of both these traces to \(\{b,d\}\) is \(\epsilon\).
\end{exa}

The definition of local observations gives one direction: global observations determine local ones. The converse, however, does not hold. As noted above, a distribution encodes independence assumptions: if two symbols $a$ and $b$ are placed in different alphabets, then $\projection{ab}{\ab_i} = \projection{ba}{\ab_i}$, for all $\ab_i \in \Omega$, and hence $\Obs_{\ab_i}(\projection{ab}{\ab_i}) = \Obs_{\ab_i}(\projection{ba}{\ab_i})$. In other words, their relative order cannot be observed locally. Any model learned over $\Omega$ must classify $ab$ and $ba$ in the same way, even if $\Obs(ab) \neq \Obs(ba)$. 

Characterizing when the local observations agree with $\Obs$ is therefore essential, as it provides a criterion for detecting when the current distribution encodes incorrect independence assumptions and must be refined.
To this end, we first define the \emph{product languages} over \(\Omega\), \ie the languages that can be represented over that distribution. 
\begin{defi}[Product language]
$\lang$ is a \emph{product language} over $\Omega$, notation $\Omega \models \lang$, iff there exists a family of languages $\{\lang_i\}_{1\leq i \leq n}$, where $\lang_i \subseteq \ab_i^\star$ for all $1\leq i\leq n$, such that $\lang = \bigpar_{i=1}^{n} (\lang_i,\ab_i)$.
\end{defi}

\begin{exa}
    \label{ex:not_model}
    In \Cref{ex:running}, it is clear by construction that \(\{\{a,b,c\},\{b,d\}\}\models \lang(\lts_1\parcomp\lts_2)\). However, \(\lang(\lts_1\parcomp\lts_2)\) is not a product language over \(\Omega_{singles}=\{\{a\},\{b\},\{c\},\{d\}\}\) because any product language over $\Omega_{singles}$ should allow for permuting $a$ and $b$ and thus, would fail to capture the fact that \(b\) can only come after one \(a\).
\end{exa}
We recall the following key lemma for product languages. 
\begin{lem}[\cite{Mukund2002}, Lemma 5.2]
    \label{lem:prod_of_projections}
    A language $\lang$ is a product language over $\Omega$ 
    if and only if $\lang = \bigpar_{i=1}^{n} (\projection{\lang}{\ab_i},\ab_i)$.  
\end{lem}
We can now define product observations over \(\Omega\), \ie, observations that are consistent with some product language over \(\Omega\).
\begin{defi}[Product observation]
    \label{def:prod_obs}
    $\Obs$ is a \emph{product observation} over $\Omega$, notation $\Omega \models \Obs$, iff there exists a language $\lang$ such that $\Omega \models \lang$ and $\lang \models \Obs$.
    We conversely say that \(\Omega\) \emph{models} \(\Obs\).
\end{defi}
While \Cref{def:prod_obs} does not prescribe how to find such a distribution given an observation function, it can be used to detect precisely when a current distribution is not consistent with observations and must be updated. 
This results in the following proposition linking local and global observations for a given distribution: an observation is a product observation over a distribution if and only if its projections according to the distribution hold the same information as the observation itself.

\begin{prop}
    \label{pr:closed_obs_distribution}
    $\Omega \models \Obs$ if and only if for all traces $\sigma \in \Dom(\Obs)$ it holds that \(\Obs(\sigma) = \bigwedge_{\ab_i \in \Omega} \Obs_{\ab_i}(\projection{\sigma}{\ab_i}) .\)
\end{prop}
\begin{proof}
    We prove each implication separately. 
    \begin{description}
        \item[$\implies$] 
    Suppose \(\Omega\models\Obs\). By definition this means that 
    there is \(\Omega\models \lang\) such that \(\lang\models \Obs\).
    Since $\lang$ is a product language, \Cref{lem:prod_of_projections} yields that 
    \begin{equation}
        \lang = \{ \sigma \in \ab^\star \mid \forall 1 \leq i \leq n \ldotp \projection{\sigma}{\ab_i} \in \lang_i \}
        \label{eq:prod_lang}
    \end{equation}
    where $\lang_i = \{ \projection{\sigma}{\ab_i} \mid \sigma \in \lang\}$ for all $i$.

    Now for any \(\sigma\in\Dom(\Obs)\), we wish to show that  
    \[\Obs(\sigma)=+\text{ iff }\bigwedge_{\ab_i \in \Omega} \Obs_{\ab_i}(\projection{\sigma}{\ab_i}) = + .\] We separate the two implications.
    If \(\Obs(\sigma)=+\), then \(\Obs_{\ab_i}(\projection{\sigma}{\ab_i})=+\) for all $i$ by definition of local observation functions and we have our result.
    On the other hand,
    if 
    \(\bigwedge_{\ab_i \in \Omega} \Obs_{\ab_i}(\projection{\sigma}{\ab_i}) = +\), 
    it means that for any $i$, \(\Obs_{\ab_i}(\projection{\sigma}{\ab_i}) = +\).
    By definition of local observations, 
    for all $i$ there is \(\sigma'_i\in\Dom(\Obs)\) such that \(\Obs(\sigma'_i)=+\) and \(\projection{\sigma}{{\ab_i}}=\projection{{\sigma'_i}}{{\ab_i}}\); 
    since \(\lang\models\Obs\), we have \(\sigma'_i\in \lang\). This implies that
    \[
        \forall i \ldotp \exists \sigma'_i \in \lang \land \projection{\sigma}{\ab_i} = \projection{{\sigma'_i}}{\ab_i}.
    \]
    Hence for all $i$, \(\projection{\sigma}{\ab_i}\in \lang_i\) by definition of \(\lang_i\). 
    It then follows from~\eqref{eq:prod_lang} that \(\sigma\in \lang\). As \(\lang\models\Obs\), we have \(\Obs(\sigma)=+\).

    \item[$\impliedby$] Suppose that \(\forall \sigma\in\Dom(\Obs),\ \Obs(\sigma)=\bigwedge_{\ab_i\in\Omega}\Obs_{\ab_i}(\projection{\sigma}{\ab_i})\). 
    For all $i$, let 
    \[
        \lang_i = \{ \sigma \in \Dom(\Obs_{\ab_i}) \mid \Obs_{\ab_i}(\sigma) = +\} \enspace .
    \]
    Now, define
    \[
        \lang=\{\sigma \in \ab^\star \mid \forall 1\leq i \leq n,\ \projection{\sigma}{\ab_i}\in \lang_i\}.
    \]
    We show that \(\lang\models\Obs\). Let \(\sigma \in \Dom(\Obs)\). Then
    \[
    \begin{aligned}
        \sigma \in \lang
        &\iff \forall 1 \leq i \leq n,\ \projection{\sigma}{\ab_i}\in \lang_i && (\text{by definition of $\lang$}) \\
        &\iff \forall 1 \leq i \leq n,\ \Obs_{\ab_i}(\projection{\sigma}{\ab_i}) = + && (\text{by definition of $\lang_i$})\\
        &\iff \Obs(\sigma) = + \enspace , && (\text{by assumption})
    \end{aligned}
    \]
    Hence \(\lang\models\Obs\), and therefore \(\Omega\models\Obs\).
    \qedhere
    \end{description}
\end{proof}
\begin{exa}
    \label{ex:modelObs}
    Following from~\Cref{,ex:not_model}, 
    consider the following observation function based on $\lang(\lts_1\parcomp\lts_2)$ from~\Cref{ex:running}: 
    \[\Obs: \epsilon\mapsto \plus;\ a\mapsto \plus;\ ab\mapsto \plus;\ b\mapsto \minus; c\mapsto \plus;\ d\mapsto \minus.\]
    Using the above proposition, we can verify that \(\Omega_{singles}=\{\{a\},\{b\},\{c\},\{d\}\}\not\models\Obs\).
    This is because \(\Obs(b)=\minus\), whereas for all $\ab_i \in \Omega_{singles}$, \(\Obs_{\ab_i}(\projection{b}{\ab_i})=\plus\) since \(\projection{ab}{\{b\}}=b\) causes \(\Obs_{\{b\}}(b)=\plus\) and for other alphabets \(\projection{b}{\ab_i}=\epsilon\) and \(\Obs_{\ab_i}(\epsilon)=+\). 
    In contrast, \(\Omega_{\{a,b\}}=\{\{a,b\},\{c\},\{d\}\}\models\Obs\) since the alphabet $\{a, b\}$ allows for distinguishing observations $b$ and $ab$.
\end{exa}
In our algorithm, this check on local and global observation functions is used to trigger an update of the current distribution exactly when necessary.

\subsection{Overview of the Approach}
\label{subsec:overview}
We present a bird's eye view of our algorithm to compositionally learn an unknown system $\SUL = \bigpar_{i=1}^{N_{\SUL}} M_i$ given only:
\begin{enumerate}
    \item[(a)] a Teacher for the whole \(\SUL\); 
    \item[(b)] knowledge of the global alphabet $\ab_{\SUL}$. 
\end{enumerate}
Crucially, we do not assume any prior knowledge of the number of components $N_{\SUL}$ and their respective alphabets. 
The architecture, depicted in \Cref{fig:comp-learn-arch}, consists of $n$ local learners $L_i$, one per component $\ab_i \in \Omega$, interacting with a standard Teacher for the global $\SUL$. This interaction is mediated by two agents: the \emph{Orchestrator}, which manages the distribution and coordinates the local learners, and the \emph{Adapter}, which translates between local and global queries.
\begin{figure}[t]
    \centering
    \begin{tikzpicture}[->,>=stealth',shorten >=0pt,auto,node distance=2.0cm,semithick]
        \node[draw,minimum height=3.6cm,minimum width=2.2cm,align=center,
        label=above:Teacher] (mat) at (8,0) {$\SUL$};
        \node[draw,minimum height=3.6cm,minimum width=1cm,align=center]
        (A) at (1.1,0) {A\\d\\a\\p\\t\\e\\r};
        \path
        ($(A.east)+(0,1)$) edge node[pos=0.6] {$\sigma \stackrel{?}{\in} \lang(\SUL)$} ($(mat.west)+(0,1)$)
        ($(A.east)+(0,-0.8)$) edge node[pos=0.6] {$\lang(\bigparallel_{i=1}^{n}\CH_i) \stackrel{?}{=} \lang(\SUL)$} ($(mat.west)+(0,-0.8)$)
        ($(mat.west)+(0,-1.2)$) edge node[pos=0.4] {CEX $\sigma$} ($(A.east)+(0,-1.2)$);
        
        \begin{scope}
        \node[align=center] (distrib) at (-3.6,3) {\(\Omega=\)\\\(\{\ab_1,\dots\ab_n\}\)};
        \node[draw,minimum size=1.4cm] (L1) at (-3.6,1.7) {$L_1$};
        \node[scale=2]                      at (-3.6,0.3) {\vdots};
        \node[draw,minimum size=1.4cm] (Ln) at (-3.6,-1.7) {$L_n$};
        \end{scope}
        
        \draw[dashed] ($(distrib.north west) + (-0.2,0.2)$) rectangle ($(Ln.south -| A.east) + (0.7,-0.3)$);
        \node at ($(distrib.north -| A.west) - (0,0.1)$) {Orchestrator};
        \node[draw,fill=white,inner sep=2.5pt] at ($(A.east) + (0.7,1)$) {};
        \node[draw,fill=white,inner sep=2.5pt] at ($(A.east) + (0.7,-0.8)$) {};
        \node[draw,fill=white,inner sep=2.5pt] at ($(A.east) + (0.7,-1.2)$) {};
        
        \path[sloped,font=\small]
        ($(L1.east)+(0,0.5)$) edge node {$\sigma_1 \stackrel{?}{\in} \projection{\lang(\SUL)}{\ab_1}$} ($(A.west)+(0,1.5)$)
        ($(L1.east)+(0,-0.4)$) edge node {$\lang(\CH_1) \stackrel{?}{=} \projection{\lang(\SUL)}{\ab_1}$} ($(A.west)+(0,0.6)$)
        ($(A.west)+(0,0.4)$) edge['] node {CEX $\sigma_1$} ($(L1.east)+(0,-0.6)$);
        
        \path[sloped,font=\small]
        ($(Ln.east)+(0,0.5)$) edge node {$\sigma_n \stackrel{?}{\in} \projection{\lang(\SUL)}{\ab_n}$} ($(A.west)+(0,-0.5)$)
        ($(Ln.east)+(0,-0.4)$) edge node {$\lang(\CH_n) \stackrel{?}{=} \projection{\lang(\SUL)}{\ab_n}$} ($(A.west)+(0,-1.4)$)
        ($(A.west)+(0,-1.6)$) edge['] node {CEX $\sigma_n$} ($(Ln.east)+(0,-0.6)$);
    \end{tikzpicture}
    \caption{
        Architecture for the compositional learning algorithm.
        The Orchestrator instantiates the Adapter and Learners according to the current distribution $\Omega$.
        It can observe queries posed by the Adapter as well as the Teacher's answer.
    }
    \label{fig:comp-learn-arch}
\end{figure}
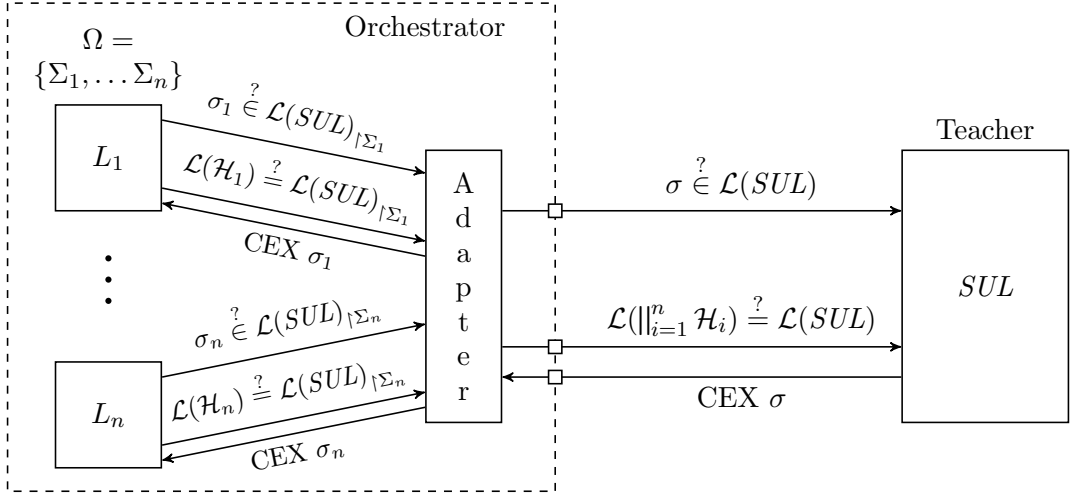

\subsubsection{Local Learners and The Adapter}\label{subsubs:adapter}
The Adapter translates membership queries posed by each local learner $L_i$ into membership queries for the global Teacher, and returns the answers back to $L_i$. Once all learners have produced a hypothesis $\CH_i$, the Adapter combines them into a global equivalence query about $\bigpar_{i=1}^{n}\CH_i$ and submits it to the Teacher. A counter-example $\sigma$ returned by the Teacher is translated back into local counter-examples $\sigma_1,\dots,\sigma_n$, which are then forwarded to the relevant learners.

Two challenges arise in this translation. First, when a local membership query contains synchronizing actions, coordination with the other components is required: the Adapter attempts to construct a global query by drawing on the observations collected by the other learners. 
This must be done in a way that is consistent with the current distribution, \ie, if \(\Omega\models \Obs\), then for any global membership query result \((\sigma,b) \in \ab^{\star}_\mathit{SUL} \times \{-,+\}\), we require \(\Omega\models \Obs\cup\{(\sigma,b)\}\). When no such query can be generated, the Adapter returns an ``unknown'' answer to the relevant learner.
Second, local observations may be misclassified. By definition of local observation functions, it may happen that initially $\Obs_{\ab_i}(\sigma) = \minus$, but later this becomes $\Obs_{\ab_i}(\sigma) = \plus$ as the global observation function is extended through counter-examples. 

\Cref{sec:local_learning} develops the Adapter and the \Lstar extensions required to make progress in the presence of unknown answers and observations that may need to be corrected.

\subsubsection{The Orchestrator}
The Orchestrator maintains the current distribution $\Omega$ of $\ab_{\SUL}$, starting from $\Omega_{\mathit{singles}} = \{\{a\} \mid a \in \ab_{\SUL}\}$. It instantiates one local learner $L_i$ per $\ab_i \in \Omega$, each targeting the local language $\projection{\lang(\SUL)}{\ab_i}$ induced by the current distribution, and a single Adapter to coordinate them. The Orchestrator observes all queries and responses exchanged between the Adapter and the Teacher, which allows it to classify counter-examples $(\sigma,b) \in \ab_{\SUL}^\star \times \{+,-\}$ as \emph{global} if $\Omega \not\models \Obs \cup \{(\sigma,b)\}$, and \emph{local} otherwise. Global counter-examples trigger a refinement of $\Omega$, upon which the Orchestrator reinstantiates the learners and the Adapter over the updated distribution. Local counter-examples are forwarded to the Adapter to update the relevant learners.

The main challenge is to refine $\Omega$ in a way that is both consistent with the new observations and guarantees progress towards a distribution that models $\lang(\SUL)$. \Cref{sec:CED} formalises this problem and shows that a suitable refinement can always be found in finitely many steps. The full algorithm, integrating the Orchestrator with the Adapter and local learners, is presented in \Cref{sec:algo}.

\section{Local Learners and The Adapter}
\label{sec:local_learning}

In this section, we focus on the Adapter and local learners, setting aside distribution refinement. To this end, we assume that a correct distribution \(\Omega=\{\ab_1,\dots,\ab_n\}\models\lang(\SUL)\) is given, \ie, that the language of $\SUL$ decomposes along \(\Omega\) as \(\lang(\SUL)=\bigpar_{\ab_i\in\Omega} (\projection{\lang(\SUL)}{\ab_i},\ab_i)\); the Orchestrator then simply instantiates the Adapter and learners once and plays no further role. Moreover, we assume that the Teacher always replies to a negative equivalence query with the \emph{shortest} counter-example $\sigma$, \ie, such that no trace shorter than $\sigma$ is a counter-example.\footnote{This assumption can be satisfied in practice by using a lexicographical ordering on the conformance test suite the Teacher generates to decide equivalence.}

We first discuss the implementation of the Adapter and show its limitations (\Cref{subsec:adapter}), then propose extensions to \Lstar to address them (\Cref{subsec:lstar_extensions}). We conclude the section with correctness results (\Cref{subsec:correctness}) and a discussion of optimisations (\Cref{subsec:optimisations}).

\subsection{Query Adapter}
\label{subsec:adapter}
The Adapter answers queries on each of the languages $\projection{\lang(\SUL)}{\ab_i}$, based on information obtained from queries on the $\SUL$. However, the application of the parallel operator causes loss of information, as the following examples illustrate.
We will use the LTSs in~\Cref{fig:running_ex} as a running example throughout this section.

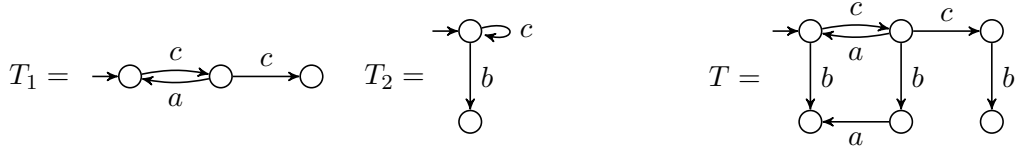
\begin{figure}[t]
	\centering
	\begin{tikzpicture}[->,>=stealth',auto,semithick]
		\def\d{1.2}
		\tikzstyle{state} = [draw,circle,inner sep=3pt];
		
		\begin{scope}[yshift=-0.5*\d cm]
			\node             at (-1.2,0) {$T_1 =$};
			\node[state] (s0) at (0,0) {};
			\node[state] (s1) at (\d,0) {};
			\node[state] (s2) at (2*\d,0) {};
			
			\path
			(-0.5,0) edge (s0)
			(s0) edge[bend left=11] node {$c$} (s1)
			(s1) edge[bend left=11] node {$a$} (s0)
			(s1) edge               node {$c$} (s2);
		\end{scope}
		
		\begin{scope}[xshift=4.5cm,yshift=-0.5*\d cm]
			\node             at (-1.0,0) {$T_2 =$};
			\node[state] (s0) at (0,0.5*\d) {};
			\node[state] (s1) at (0,-0.5*\d) {};
			
			\path
			(-0.5,0.5*\d) edge (s0)
			(s0) edge[loop right] node {$c$} (s0)
			(s0) edge             node {$b$} (s1);
		\end{scope}
		
		\begin{scope}[xshift=9cm]
			\node             at (-1.0,-0.5*\d) {$T =$};
			\node[state] (s0) at (0,0) {};
			\node[state] (s1) at (\d,0) {};
			\node[state] (s2) at (2*\d,0) {};
			\node[state] (t0) at (0,-\d) {};
			\node[state] (t1) at (\d,-\d) {};
			\node[state] (t2) at (2*\d,-\d) {};
			
			\path
			(-0.5,0) edge (s0)
			(s0) edge[bend left=11] node {$c$} (s1)
			(s1) edge[bend left=11] node {$a$} (s0)
			(s1) edge               node {$c$} (s2)
			(t1) edge               node {$a$} (t0)
			(s0) edge               node {$b$} (t0)
			(s1) edge               node {$b$} (t1)
			(s2) edge               node {$b$} (t2);
		\end{scope}
	\end{tikzpicture}
	\caption{Running example consisting of two LTSs $T_1$ and $T_2$, their parallel composition $T$.}
	\label{fig:running_ex}
\end{figure}
\begin{exa}
	Consider the LTSs $T_1$, $T_2$ and $T = T_1 \parallel T_2$ depicted in \Cref{fig:running_ex}, with alphabets $\{a,c\}$, $\{b,c\}$ and $\{a,b,c\}$ respectively, and the distribution \(\Omega=\{\{a,c\},\{b,c\}\}\).
	
	Suppose the Teacher answers $bc \notin \lang(T)$ to a membership query. By definition of parallel composition, $bc \notin \lang(T)$ if and only if $\projection{bc}{\{a,c\}} = c \notin \lang(T_1)$ or $\projection{bc}{\{b,c\}} = bc \notin \lang(T_2)$, so this single observation does not allow us to determine which case holds; here only the latter does. Similarly, if a composite hypothesis $\CH = \CH_1 \parallel \CH_2$ is rejected with counter-example $ccc \notin \lang(T)$, we cannot determine whether $ccc \notin \lang(T_1)$, $ccc \notin \lang(T_2)$, or both; here both are true.
\end{exa}
Generally, given negative information on the composite level ($\sigma \notin \lang(\SUL)$), it is hard to infer information for a single component, whereas positive information ($\sigma \in \lang(\SUL)$) easily translates back to the level of individual components.

We thus relax the guarantees on the answers given by the Adapter in the following way:
\begin{enumerate}
	\item Not all membership queries can be answered, the Adapter may return the answer `unknown'.
	\item An equivalence query for component $i$ can be answered with a \emph{spurious} counter-example $\sigma_i \in \lang(\CH_i) \cap \projection{\lang(\SUL)}{\ab_i}$, \ie, a trace that belongs to $\projection{\lang(\SUL)}{\ab_i}$ but is returned as a negative counter-example.
\end{enumerate}
The procedures that implement the Adapter are stated in~\cref{alg:comp_i_queries_basic}.
For each $1 \leq i \leq n$, we have one instance of each of the functions $\Member_i$ and $\Equiv_i$, used by the $i$th learner to pose its queries.
Here, we assume that for each component $i$, a copy of the latest hypothesis $\CH_i$ is stored, as well as a representation of \(\Obs_{\ab_i}^+=\{\sigma_i\mid(\sigma_i,+)\in\Obs_{\ab_i}\}\), the traces that we know to be in $\projection{\lang(\SUL)}{\ab_i}$, populated
incrementally from positive answers to membership queries (line~\ref{line:adap_expandT_mem}) and positive counter-examples (line~\ref{line:adap_expandT_equiv}).
Membership and equivalence queries on $\SUL$ will be forwarded to the Teacher via the functions $\Member(\sigma)$ and $\Equiv(\CH)$, respectively.

\subsubsection{Membership Queries}
A membership query $\sigma \in \projection{\lang(\SUL)}{\ab_i}$ can be answered directly by posing $\sigma \in \lang(\SUL)$ to the Teacher if $\sigma$ contains only actions local to $\ab_i$.
However, in the case where $\sigma$ contains synchronizing actions, cooperation from other components is required.
To this end, the sets $\Obs_{\ab_j}^+$ are composed into \(\Obs_{\cup_{j\neq i}\ab_j}^+=\bigpar_{j \neq i}(\Obs_{\ab_j}^+,\ab_j) \subseteq \projection{\lang(\SUL)}{\bigcup_{j\neq i}\ab_j}\), an under-approximation of the behavior of other components, possibly including some synchronizing actions they can perform.

From this, we extract the subset $\Pi_i$ (line~\ref{line:adap_Pi}) of traces whose synchronizing actions with component $i$ match those required by $\sigma$: 
\[
\Pi_i=\{\sigma'\in\Obs_{\cup_{j\neq i}\ab_j}^+\mid\projection{\sigma'}{\ab_i}=\projection{\sigma}{\cup_{j\neq i}\ab_j}\} \enspace .
\]
If \(\Pi_i\) is empty, we do not have sufficient information on how other components can cooperate, and the Adapter returns `unknown' (line~\ref{line:adap_return_qm}) to the membership query. 
Else, the Adapter takes \(\sigma'\in\Pi_i\), constructs an interleaving \(\sigma_{int}\in(\{\sigma\},\ab_i)\bigpar (\{\sigma'\},\cup_{j\neq i}\ab_j)\) and forwards it to the Teacher (line~\ref{line:adap_return_member}).

In particular, if $\sigma$ contains no synchronizing actions, then $\Pi_i$ is non-empty, as it contains at least $\epsilon$.

\begin{exa}
	Refer to the running example in Figure~\ref{fig:running_ex}.
	Suppose that the current knowledge about $\lts_2$ is $\Obs_{\ab_2}^+ = \{\epsilon,b\}$.
	When $\Member_1(c)$ is called, $\Pi_i = \emptyset$, because there is no trace $\sigma' \in \Obs_{\ab_2}^+$ that is equal to $c$ when restricted to $\{a,c\}$, therefore $\mathit{unknown}$ is returned.
	Intuitively, since the second learner has not yet discovered that $c$ or $bc$ (or some other trace containing a $c$) is in its language, the Adapter is unable to turn the query $c$ on $\lts_1$ into a query for the composite system.
	\qed
\end{exa}
\begin{exa}
	Suppose now that $cac \in \Obs_{\ab_1}^+$, \ie, we already learned that $cac \in \lang(\lts_1)$.
	When posing the membership query $cbc \in \lang(\lts_2)$, the Adapter finds that $cac$ and $cbc$ contain the same synchronizing actions (\viz $cc$) and constructs an interleaving, for example $cabc$.
	The Teacher answers negatively to the query $cabc \in \lang(\lts)$, and thus we learn that $cbc \notin \lang(\lts_2)$.
	\qed
\end{exa}

\subsubsection{Equivalence Queries}
For equivalence queries, the Adapter offers functions $\Equiv_i$.
To construct a corresponding query on the composite level, we first need to gather a hypothesis $\CH_i$ for each $i$.
Thus, we synchronize all learners in a barrier (line~\ref{line:adap_barrier}), after which a composite hypothesis can be constructed and forwarded to the Teacher (lines~\ref{line:adap_construct_H}, \ref{line:adap_equiv_query}).
An affirmative answer can be returned directly, while in the negative case we investigate the returned counter-example $\sigma$.
Since we assumed that $\sigma$ is shortest, 
$\CH$ and $\projection{\lang(\SUL)}{\ab_i}$ agree on all $\sigma' \in \pref(\sigma) \setminus \{\sigma\}$.
Thus, $\sigma$ only concerns $\CH_i$ if the last action in $\sigma$ is contained in $\ab_i$.
Furthermore, we need to check whether $\CH$ and $\CH_i$ agree on $\sigma$: it can happen that $\projection{\sigma}{\ab_i} \in \lang(\CH_i)$ but $\sigma \notin \lang(\CH)$ due to other hypotheses not providing the necessary communication opportunities.
If both conditions are satisfied (line~\ref{line:adap_check_appl}), we return the $\projection{\sigma}{\ab_i}$ (line~\ref{line:adap_return_no}).
Otherwise, we cannot conclude anything about $\CH_i$ at this moment and we iterate (line~\ref{line:adap_repeat}).
In that case, we effectively wait for other hypotheses $\CH_j$, with $j \neq i$, to be updated before trying again.
A termination argument is provided later in this section.

\begin{algorithm}[t]
	\caption{Membership and equivalence query procedures for component $i$.}
	\label{alg:comp_i_queries_basic}

	\SetVlineSkip{2pt}

	\SetKwProg{Fn}{Function}{}{end}

	\KwIn{A distribution $\Omega=\{\ab_1,\dots,\ab_n\}$.}
	\KwData{for each $i$, the latest hypothesis $\CH_i$ and $\Obs^+_{\ab_i}$.
	}

	\Fn{$\Member_i(\sigma)$}{
		$\Pi_i := \{ \sigma' \in \Obs_{\cup_{j\neq i}\ab_j}^+ \mid \projection{\sigma'}{\ab_i} = \projection{\sigma}{\cup_{j \neq i} \ab_j} \}$ \; \label{line:adap_Pi}
		\uIf{$\Pi_i \neq \emptyset$}{
			$\mathit{answer} := \Member(\sigma_\mathit{int})$ for some $\sigma' \in \Pi_i$ and $\sigma_\mathit{int} \in (\{\sigma\},\ab_i)\bigpar (\{\sigma'\},\cup_{j\neq i}\ab_j)$  \tcc*{construct interleaving} \label{line:adap_return_member}
			\lIf{$\mathit{answer} = yes$}{$\Obs^+_{\ab_i} := \Obs^+_{\ab_i} \cup \{\sigma\}$\label{line:adap_expandT_mem}}
			\Return{answer}
		}
		\lElse{\Return{unknown}\label{line:adap_return_qm}}
	}
	\Fn{$\Equiv_i(\CH')$}{
		$\CH_i := \CH'$\;
		\While{true}{\label{line:adap_repeat}
			barrier$(n)$ \tcc*{wait until this point is reached for every i} \label{line:adap_barrier}%
			construct $\CH = \bigpar_i \CH_i$\; \label{line:adap_construct_H}
			\Switch{$\Equiv(\CH)$}{ \label{line:adap_equiv_query}
				\lCase{yes}{\Return{yes}}
				\Case{$(no, \sigma)$}{
					\lIf{$\sigma \notin \lang(\CH)$}{$\Obs^+_{\ab_i} := \Obs^+_{\ab_i} \cup \{\projection{\sigma}{\ab_i}\}$\label{line:adap_expandT_equiv}}
					\If{$a \in \ab_i$, \textnormal{where} $\sigma = \sigma'a$, \textnormal{and} $\sigma \in \lang(\CH) \Leftrightarrow \projection{\sigma}{\ab_i} \in \lang(\CH_i)$}{\label{line:adap_check_appl}
						\Return{$(no,\projection{\sigma}{\ab_i})$} \label{line:adap_return_no}
					}
				}
			}
		}
	}
\end{algorithm}
\begin{exa}
	\label{ex:pos_cex}
	Again considering our running example (Figure~\ref{fig:running_ex}), suppose the two learners call in parallel the functions $\Equiv_1(\CH_1)$ and $\Equiv_2(\CH_2)$.
	The provided hypotheses and their parallel composition are as follows:
	\[
	\CH_1 =
	
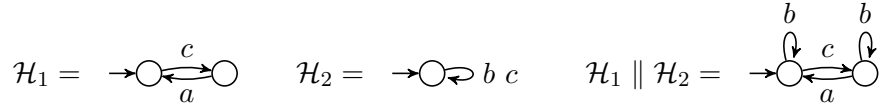
\tikz[baseline=-0.65ex, ->,>=stealth',auto,semithick, initial text={}, initial distance={10pt}, every state/.style={inner sep=1pt, minimum size=2ex}]{
		\node[state, initial] (s0) at (0,0) {};
		\node[state]          (s1) at (1,0) {};
		\path
		(s0) edge[bend left=12] node {$c$} (s1)
		(s1) edge[bend left=12] node {$a$} (s0);
	}
	\qquad
	\CH_2 =
	\tikz[baseline=-0.65ex, ->,>=stealth',auto,semithick, initial text={}, initial distance={10pt}, every state/.style={inner sep=1pt, minimum size=2ex}]{
		\node[state, initial] (s0) at (0,0) {};
		\path
		(s0) edge[loop right] node {$b$ $c$} (s0);
	}
	\qquad
	\CH_1 \parallel \CH_2=
	\tikz[baseline=-0.65ex, ->,>=stealth',auto,semithick, initial text={}, initial distance={10pt}, every state/.style={inner sep=1pt, minimum size=2ex}]{
		\node[state, initial] (s0) at (0,0) {};
		\node[state]          (s1) at (1,0) {};
		\path
		(s0) edge[loop above]   node {$b$} (s0)
		(s0) edge[bend left=12] node {$c$} (s1)
		(s1) edge[bend left=12] node {$a$} (s0)
		(s1) edge[loop above]   node {$b$} (s1);
	}
	\]
	The Adapter forwards $\CH = \CH_1 \parallel \CH_2$ to the Teacher, which returns the counter-example $cc$.
	The last symbol, $c$, occurs in both alphabets, but $cc \in \lang(\CH)$ does not hold and $\projection{cc}{\ab_2} \in \lang(\CH_2)$ does, so only $\Equiv_1(\CH_1)$ returns $(\mathit{no},cc)$.
	The call to $\Equiv_2(\CH_2)$ hangs in the while loop of line~\ref{line:adap_repeat} until $\Equiv_1$ is invoked with a different hypothesis.
\end{exa}
\begin{exa}
	\label{ex:spurious_cex}
	Suppose now that the hypotheses and their composition are:
	\[
	\CH_1 =
	
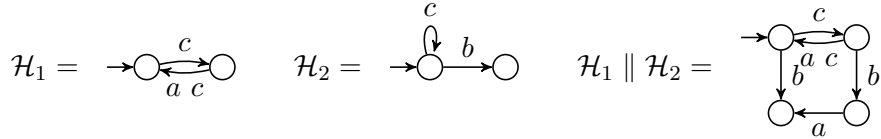
\tikz[baseline=-0.65ex, ->,>=stealth',auto,semithick, initial text={}, initial distance={10pt}, every state/.style={inner sep=1pt, minimum size=2ex}]{
		\node[state, initial] (s0) at (0,0) {};
		\node[state]          (s1) at (1,0) {};
		\path
		(s0) edge[bend left=12] node {$c$} (s1)
		(s1) edge[bend left=12] node {$a$ $c$} (s0);
	}
	\qquad
	\CH_2 =
	\tikz[baseline=-0.65ex, ->,>=stealth',auto,semithick, initial text={}, initial distance={10pt}, every state/.style={inner sep=1pt, minimum size=2ex}]{
		\node[state, initial] (s0) at (0,0) {};
		\node[state]          (s1) at (1,0) {};
		\path
		(s0) edge[loop above]   node {$c$} (s0)
		(s0) edge               node {$b$} (s1);
	}
	\qquad
	\CH_1 \parallel \CH_2=
	\tikz[baseline=-0.5cm, ->,>=stealth',auto,semithick, initial text={}, initial distance={10pt}, every state/.style={inner sep=1pt, minimum size=2ex}]{
		\node[state, initial] (s0) at (0,0) {};
		\node[state]          (s1) at (1,0) {};
		\node[state]          (s2) at (0,-1) {};
		\node[state]          (s3) at (1,-1) {};
		\path
		(s0) edge[bend left=12] node {$c$} (s1)
		(s1) edge[bend left=12] node {$a$ $c$} (s0)
		(s0) edge               node {$b$} (s2)
		(s1) edge               node {$b$} (s3)
		(s3) edge               node {$a$} (s2)
		;
	}
	\]
	When we submit $\Equiv(\CH_1 \parallel \CH_2)$, we may receive the negative counter-example $ccc$, which is a shortest counter-example.
	This counter-example does not contain any information to suggest that it only applies to $\CH_1$.
	It is a spurious counter-example for $\CH_2$, since that \emph{should} contain the trace $ccc$.
\end{exa}

\subsection{\Lstar extensions}
\label{subsec:lstar_extensions}
As explained in the previous section, the capabilities of our Adapter are limited compared to an ordinary Teacher.
We thus extend \Lstar to deal with the answer `unknown' to membership queries and to deal with spurious counter-examples.

\subsubsection{Answer `unknown'.}
The setting of receiving incomplete information through membership queries first occurred  in the work of Grinchtein, Leucker and Pieterman~\cite{GrinchteinLP06}, and is also discussed by Leucker and Neider~\cite{LeuckerN12}.
Here we briefly recall the original ideas~\cite{GrinchteinLP06}.
To deal with partial information from membership queries, the concept of an observation table is generalized such that the function $T : (S \cup S \cdot \ab) \cdot E \to \{\plus,\minus\}$ is a partial function, that is, for some cells we have no information.
Based on $T$, we now define the function $\row : S \cup S \cdot \ab \to E \to \{\plus,\minus,\qm\}$ to fill the cells of the table: $\row_T(s)(e) = T(se)$ if $T(se)$ is defined and $\qm$ otherwise.
We refer to `\qm' as a \emph{wildcard}; its actual value is currently unknown and might be learned at a later time or never at all.
To deal with the uncertain nature of wildcards, we introduce a relation $\weakeq$ on rows, where $\row(s_1) \weakeq \row(s_2)$ iff for every $e \in E$, $\row(s_1)(e) \neq \row(s_2)(e)$ implies that $\row(s_1)(e) = \qm$ or $\row(s_2)(e) = \qm$.
Note that $\weakeq$ is not an equivalence relation since it is not transitive.
Closedness and consistency are defined as before, but now use the new relation $\weakeq$.
We say an LTS $M$ is \emph{consistent} with $T$ iff for all $s \in \ab^\star$ such that $T(s)$ is defined, we have $T(s) = \plus$ iff $s \in \lang(M)$.

As discussed earlier, Angluin's original \Lstar algorithm relies on the fact that, for a closed and consistent table, there is a unique minimal DFA (or, in our case, LTS) that is consistent with $T$.
However, the occurrence of wildcards in the observation table may allow multiple minimal LTSs that are consistent with $T$.
Such a minimal consistent LTS can be obtained with a SAT solver, as described in~\cite{HeuleV10}.

Similar to Angluin's original algorithm, this extension comes with some correctness theorems.
First of all, it terminates outputting the minimal LTS for the target language.
Furthermore, each hypothesis is consistent with all membership queries and counter-examples that were provided so far.
Lastly, each subsequent hypothesis has at least as many states as the previous one, but never more than the minimal LTS for the target language~\cite{LeuckerN12}.

\subsubsection{Spurious Counter-Examples.}
We now extend this algorithm with the ability to deal with spurious counter-examples.
Any \emph{negative} counter-example $\sigma \in \lang(\CH_i)$ might be spurious, \ie, it is actually the case that $\sigma \in \projection{\lang(\SUL)}{\ab_i}$: this could happen when a positive global word projecting on \(\sigma\) exists but has yet to be observed.
Since \Lstar excludes $\sigma$ from the language of all subsequent hypotheses, we might later get the same trace $\sigma$, but now as a \emph{positive} counter-example.
In that case, the initial negative judgement from the equivalence Teacher was spurious.

\begin{algorithm}[t]
	\caption{Learning with wildcards and backtracking.}
	\label{alg:lstar_backtracking}

	Set $BT$ to $\emptyset$\; \label{line:lstarqb_bt_init}
	Initialize $S$ and $E$ to $\{\epsilon\}$\;
	Extend $T$ to $S \cup S \cdot \ab_i$\ by calling $\Member_i$\;
	\Repeat{Teacher replies \emph{yes} to conjecture $\CH$}{
		\While{$(S,E,T)$ is not closed and consistent}{\label{line:lstarqb_cc1}
			\If{$(S,E,T)$ is not consistent}{
				Find $s_1, s_2 \in S$, $a \in \ab_i$, $e \in E$ such that $\mathit{row}_{T}(s_1) \weakeq \mathit{row}_{T}(s_2)$ and $T(s_1 \cdot a \cdot e) \not\weakeq T(s_2 \cdot a \cdot e)$\;
				Add $a \cdot e$ to $E$ and extend $T$ by calling $\Member_i$\;
			}
			\If{$(S,E,T)$ is not closed}{
				Find $s_1 \in S$, $a \in \ab_i$ such that $\mathit{row}_{T}(s_1 \cdot a) \not \weakeq \mathit{row}_{T}(s)$ for all $s \in S$\;
				Add $s_1 \cdot a$ to $S$ and extend $T$ by calling $\Member_i$\;\label{line:lstarqb_cc2}
			}
		}
		Call $\Equiv_i(\CH)$ for some minimal LTS $\CH$ consistent with $T$\; \label{line:lstarqb_construct_H}
		\If{Teacher replies with counter-example $\sigma$}{
			\If(\tcc*[f]{$\sigma$ corrects an earlier spurious CEX}){$T(\sigma) = \minus$}{
				$(S,E,T) := BT(\sigma)$\; \label{line:lstarqb_backtrack1}
			}
			\ElseIf(\tcc*[f]{$\sigma$ might be spurious}){$\sigma \in \lang(\CH)$}{
				$BT(\sigma) := (S,E,T)$\; \label{line:lstarqb_backtrack2}
			}
			Add $\sigma$ and all its prefixes to $S$ and extend $T$ by calling $\Member_i$\; \label{line:lstarqb_add_cex}
		}
	}
	\Return{$\CH$}\;
\end{algorithm}

One possible way of dealing with spurious counter-examples, is adding to $\Lstar$ the ability to \emph{overwrite} entries in the observation table in case a spurious counter-example is corrected.
However, this may cause the learner to diverge if infinitely many spurious counter-examples are returned.
Therefore, we instead choose to add a backtracking mechanism to ensure our search will converge.
The pseudo code is listed in~\cref{alg:lstar_backtracking}; we refer to this as $\Lstarqb$ (\Lstar with wildcards and backtracking).

We have a mapping $BT$ that stores backtracking points; $BT$ is initialized to the empty mapping (line~\ref{line:lstarqb_bt_init}).
Lines~\ref{line:lstarqb_cc1}-\ref{line:lstarqb_cc2} ensure the observation table is closed and consistent in the same way as \Lstar, but use the relation $\weakeq$ on rows instead.
Next, we construct a minimal hypothesis that is consistent with the observations in $T$ (line~\ref{line:lstarqb_construct_H}).
This hypothesis is posed as an equivalence query.
If the Teacher replies with a counter-example $\sigma$ for which $T(\sigma) = \minus$, then $\sigma$ was a spurious counter-example, so we backtrack and restore the observation table from just before $T(\sigma)$ was introduced (line~\ref{line:lstarqb_backtrack1}).
Otherwise, we store a backtracking point for when $\sigma$ later turns out to be spurious (line~\ref{line:lstarqb_backtrack2}); this is only necessary if $\sigma$ is a negative counter-example.
Note that not all information is lost when backtracking: the set $\Obs_{\ab_i}^+$ stored in the Adapter is unaffected, so some positive traces are carried over after backtracking.
Finally, we incorporate $\sigma$ into the observation table (line~\ref{line:lstarqb_add_cex}).
When the Teacher accepts our hypothesis, we terminate.

We conclude this section with an example that shows how spurious counter-examples can be resolved.

\begin{exa}
	Refer again to the LTSs of our running example in Figure~\ref{fig:running_ex}.
	Consider the situation after proposing the hypotheses of Example~\ref{ex:spurious_cex} and receiving the counter-example $ccc$, which is spurious for the second learner.

	In the next iteration, $\Member_2$ can answer some membership queries, such as $cbc$, necessary to expand the table of the second learner.
	This is enabled by the fact that $\Obs^+_{\ab_1}$ contains $cc$ from the positive counter-example of Example~\ref{ex:pos_cex} (line~\ref{line:adap_Pi} of~\cref{alg:comp_i_queries_basic}).
	The resulting updated hypotheses are as follows.
	
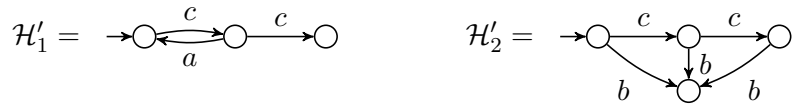
\begin{center}
		\begin{tikzpicture}[->,>=stealth',auto,semithick]
			\def\d{1.2}
			\tikzstyle{state} = [draw,circle,inner sep=3pt];

			\begin{scope}
				\node             at (-1.3,0) {$\CH'_1 =$};
				\node[state] (s0) at (0,0) {};
				\node[state] (s1) at (\d,0) {};
				\node[state] (s2) at (2*\d,0) {};

				\path
				(-0.5,0) edge (s0)
				(s0) edge[bend left=11] node {$c$} (s1)
				(s1) edge[bend left=11] node {$a$} (s0)
				(s1) edge               node {$c$} (s2);
			\end{scope}

			\begin{scope}[xshift=6cm]
				\node             at (-1.3,0) {$\CH'_2 =$};
				\node[state] (s0) at (0,0)    {};
				\node[state] (s1) at (\d,0)   {};
				\node[state] (s2) at (2*\d,0) {};
				\node[state] (s3) at (\d,-0.6*\d) {};

				\path
				(-0.5,0) edge (s0)
				(s0) edge             node {$c$} (s1)
				(s1) edge             node {$c$} (s2)
				(s0) edge[',bend right=10]   node {$b$} (s3)
				(s1) edge             node {$b$} (s3)
				(s2) edge[bend left=10] node {$b$} (s3);
			\end{scope}
		\end{tikzpicture}
	\end{center}
	Now the counter-example to composite hypothesis $\CH'_1 \parallel \CH'_2$ is $cacc$.
	The projection on $\ab_2$ is $ccc$, which directly contradicts the counter-example received in the previous iteration.
	This spurious counter-example is thus repaired by backtracking in the second learner.
	The invocation of $\Equiv_1(\CH'_1)$ by the first learner does not return this counter-example, since $\CH'_1 \parallel \CH'_2$ and $\CH'_1$ do not agree on $cacc$, so the check on line~\ref{line:adap_check_appl} of~\cref{alg:comp_i_queries_basic} fails.

	Finally, in the next iteration, the respective hypotheses coincide with $\lts_1$ and $\lts_2$ and both learners terminate.
\end{exa}

\subsection{Correctness}
\label{subsec:correctness}

As a first result, we show that our Adapter provides correct information on each of the components when asking membership queries.
This is required to ensure that information obtained by membership queries does not conflict with counter-examples.
\begin{thm}
	Answers from $\Member_i$ are consistent with $\projection{\lang(\SUL)}{\ab_i}$.
\end{thm}
\begin{proof}
	We consider a query $\Member_i(\sigma)$ for component $i$.
	The answer \emph{unknown} is trivially consistent with $\projection{\lang(\SUL)}{\ab_i}$, so we assume that the set $\Pi_i$ is not empty.
	By construction, each set $\Obs_{\ab_i}^+$ verifies $\Obs_{\ab_i}^+ \subseteq \projection{\lang(\SUL)}{\ab_i}$.
	Thus, it also holds that $\bigpar_{j \neq i} (\Obs_{\ab_j}^+,\ab_j) \subseteq \projection{\lang(\SUL)}{\bigcup_{j \neq i}\ab_j}$ and any $\sigma' \in \Pi_i$ occurs in $\projection{\lang(\SUL)}{\bigcup_{j \neq i}\ab_j}$.
	Since $\sigma$ and $\sigma'$ correctly synchronize on shared actions, for any interleaving $\sigma_\mathit{int} \in (\{\sigma\},\ab_i)\bigpar (\{\sigma'\},\cup_{j\neq i}\ab_j)$, we have that $\sigma \in \projection{\lang(\SUL)}{ \ab_i}$ if and only if $\sigma_\mathit{int} \in \lang(\SUL)$.
\end{proof}

Before presenting the main theorem on correctness of our learning framework, we first introduce several auxiliary lemmas.
In the following, we assume $n$ instances of \Lstarqb run concurrently and each queries the corresponding functions $\Member_i$ and $\Equiv_i$, as per our architecture (Figure~\ref{fig:comp-learn-arch}).
First, a counter-example cannot be spurious for all learners; thus at least one learner obtains valid information to progress its learning.

\begin{lem}
	\label{lem:cex_valid_for_one}
	Every counter-example obtained from $\Equiv(\CH)$ is valid for at least one learner.
\end{lem}
\begin{proof}
	We distinguish the cases where $\sigma$ is a positive or negative counter-example.
	If $\sigma$ is positive (that is, $\sigma \in \lang(\SUL)$ and $\sigma \notin \lang(\CH)$), then by the definition of projection, $\projection{\sigma}{\ab_i} \in \projection{\lang(\SUL)}{\ab_i}$ for every $i$.
	By $\sigma \notin \lang(\CH)$ and the definition of parallel composition there must be at least one $i$ such that $\projection{\sigma}{\ab_i} \notin \lang(\CH_i)$; $\projection{\sigma}{\ab_i}$ is valid for such $i$.

	\smallskip
	
	For the case where $\sigma \notin \lang(\SUL)$, \(\sigma\in\lang(\CH)\) as it is a counter-example. Hence for all $i$, $\projection{\sigma}{\ab_i} \in \lang(\CH_i)$ by definition of the parallel composition. 

	We reason by contradiction and assume that $\projection{\sigma}{\ab_i}$ is spurious for all $i$.
	That is, it holds that $\projection{\sigma}{\ab_i} \in \lang(\CH_i)$ and $\projection{\sigma}{\ab_i} \in \projection{\lang(\SUL)}{\ab_i}$.
	This, however, directly contradicts that $\sigma \notin \lang(\SUL)$.
\end{proof}

The next lemma shows that even if a spurious counter-example occurs, this does not induce divergence, since it is always repaired by a corresponding positive counter-example in finite time.

\begin{lem}
	\label{lem:spurious_cex_finite_repair}
	If $\Equiv(\CH)$ always returns a shortest counter-example, then each spurious counter-example is repaired by another counter-example within a finite number of invocations of $\Equiv(\CH)$, the monolithic Teacher.
\end{lem}
\begin{proof}
	Suppose at some point a spurious counter-example $\projection{\sigma}{\ab_i}$ is returned to the $i$th learner. Consider the set 
	\[
		S = \{\sigma' \in \lang(\SUL) \mid \projection{\sigma'}{\ab_i} = \projection{\sigma}{\ab_i}\} \enspace , 
	\]
	which contains all traces that may repair $\projection{\sigma}{\ab_i}$. Since $\projection{\sigma}{\ab_i}$ is spurious, $S$ is not empty. Subsequent local hypotheses $\CH_i$ respect the spurious counter-example, so $\projection{\sigma}{\ab_i} \notin \lang(\CH_i)$ for all those $\CH_i$, and $S \cap \lang(\CH) = \emptyset$ for subsequent composite hypotheses $\CH$. The traces in $S$ are thus positive counter-examples for all subsequent composite hypotheses. 
	
	Let \(m = \min \{ |\sigma'| \mid \sigma' \in S \}\), and define
	\[
		W = \left(\bigcup_{k = 0}^{m} \ab^k\right) \setminus S \enspace .
	\]
	Thus, \(W\) is the finite set of traces that are strictly shorter than every trace in \(S\). Since every trace in \(S\) is a subsequent counter-example and \(\Equiv(H)\) returns a shortest counter-example, any subsequent counter-example is either in \(S\) or shorter than every trace in \(S\), hence belongs to \(W\). Therefore, it suffices to show that only finitely many counter-examples from $W$ are returned before one from $S$ appears. 
	
	Write $\lang_j = \projection{\lang(\SUL)}{\ab_j}$ and $W_j = \projection{W}{\ab_j} \cap \lang_j$, for each $j \in \{1,\dots,n\}$. 
	Counter-examples $\sigma \in W$ that trigger a backtrack in learner $j$ permanently add $\projection{\sigma}{\ab_j}$ to $\Obs_{\ab_j}^+$, so there are at most $|W_j|$ backtracks. Since each backtrack reverts table entries and can thus invalidate at most $|W_j|$ spurious entries, the total number of counter-example receptions due to backtracks and spurious entries is finite.

	Counter-examples $\sigma \in W$ that trigger neither a backtrack nor a spurious entry cause the hypothesis of at least one learner to grow (Lemma~\ref{lem:cex_valid_for_one}). Between consecutive backtracks, table entries are monotone, so by~\cite[Lemma 3]{LeuckerN12} the hypothesis size of learner $j$ is non-decreasing throughout each such interval and bounded by $n_j + |W_j|$, where $n_j$ is the Nerode index of $\lang_j$ and the additive term accounts for spurious negative entries that may incorrectly split Nerode classes. The total number of such counter-examples is therefore finite.

	Since both cases contribute finitely many counter-examples from $W$, $\Equiv(\CH)$ must eventually return some $\sigma' \in S$, repairing $\projection{\sigma}{\ab_i}$.
\end{proof}

Our main theorem states that a composite system is learned by $n$ copies of \Lstarqb that each call our Adapter (see Figure~\ref{fig:comp-learn-arch}).

\begin{thm}
	\label{thm:fixed_dist}
	Running $n$ instances of \Lstarqb synchronized by the Adapter terminates, and on termination we have $\lang(\CH_1 \parallel \dots \parallel \CH_n) = \lang(\SUL)$.
\end{thm}
\begin{proof}[Sketch]
	Since each of the $n$ learners always eventually proposes a hypothesis (as long as they have not terminated yet), they will also repeatedly synchronize properly in the barrier (line~\ref{line:adap_barrier} of~\cref{alg:comp_i_queries_basic}) and thus no deadlock can occur.
	
	For each learner $i$, we can now reason as follows.
	As long as no spurious counter-example occurs for this learner, it makes progress as usual~\cite[Lemma 3]{LeuckerN12}.
	Suppose learner $i$ receives a spurious counter-example $\sigma$ after proposing a hypothesis $\CH_i$.
	By Lemma~\ref{lem:spurious_cex_finite_repair}, \(\sigma\) is repaired after finitely many invocations of \(\Equiv(H)\).
	At that point, learner $i$ may propose $\CH_i$ again, since $\sigma \in \lang(\CH_i)$, by the fact that it was spurious initially.
	Hence, the only way for learner $i$ to avoid eventual progress is to immediately receive a spurious counter-example again, which is fixed in finite time, proposing $\CH_i$ again and so on indefinitely.
	However, by \Cref{lem:cex_valid_for_one}, every counter-example spurious for learner $i$ is valid for some other learner, so progress is still being made elsewhere. Since there are only finitely many learners, at least one learner $j$ will eventually converge to $\lang(\CH_j) = \projection{\lang(\SUL)}{\ab_j}$.
	
	We conclude by induction on the number of learners that have converged.
	Once learner $j$ has converged, \ie, $\lang(\CH_j) = \projection{\lang(\SUL)}{\ab_j}$, no counter-examples are valid for it, so some other learners for which $\lang(\CH_k) \neq \projection{\lang(\SUL)}{\ab_k}$ will receive valid counter-examples.
	By induction, they will all eventually converge. Then, $\Equiv(\CH)$ will respond with \emph{yes} (since $\lang(\SUL)$ is a product language over $\Omega$) and we terminate.
\end{proof}

\begin{rem}
	\label{rem:matching_local_languages}
	Even when $\SUL = \bigpar_{i=1}^n M_i$ with each $M_i$ over $\ab_i$, \Cref{thm:fixed_dist} does not imply $\CH_i = M_i$ for all $i$, since we only observe the behavior of $M_i$ in the context of the parallel composition. 
	For example, consider the LTSs below, both with alphabet $\{a\}$:
	\[
	T_1 =
	\tikz[baseline=-0.65ex, ->,>=stealth',auto,semithick, initial text={}, initial distance={10pt}, every state/.style={inner sep=1pt, minimum size=2ex}]{
		\node[state, initial] (s0) at (0,0) {};
	}
	\qquad
	T_2 =
	\tikz[baseline=-0.65ex, ->,>=stealth',auto,semithick, initial text={}, initial distance={10pt}, every state/.style={inner sep=1pt, minimum size=2ex}]{
		\node[state, initial] (s0) at (0,0) {};
		\path
		(s0) edge[loop right] node {$a$} (s0);
	}
	\]
	Here $T_1 \parallel T_2 = T_1 \parallel T_1$, so the equivalence oracle cannot distinguish the two decompositions.
\end{rem}

\subsection{Optimisations}
\label{subsec:optimisations}
There are a number of optimisations that can dramatically improve the practical performance of our learning framework.
We briefly discuss them here.

First, finding whether there is a trace $\sigma' \in \Pi_i$ (line~\ref{line:adap_Pi} of~\cref{alg:comp_i_queries_basic}) can quickly become expensive once the sets $\Obs_{\ab_i}^+$ grow larger.
We thus try to limit the size of each $\Obs_{\ab_i}^+$ without impacting the amount of information it provides on the synchronization opportunities offered by component $i$.
Therefore, when we derive that $\sigma \in \projection{\lang(\SUL)}{\ab_i}$, we only store the shortest prefix $\rho$ of $\sigma$ such that $\rho$ and $\sigma$ contain the same synchronizing actions.
That is, $\sigma = \rho \cdot \rho'$ and $\rho'$ contains only actions local to $\ab_i$ in \(\Omega\).
Furthermore, we construct $\Obs_{\cup_{j\neq i}\ab_j}^+$ only once after each call to $\Equiv_i$ and we cache accesses to $\Obs_{\cup_{j\neq i}\ab_j}^+$, such that it is only traversed once when performing multiple queries $\sigma^1$, $\sigma^2$ for which it holds that $\projection{\sigma^1}{\cup_{j\neq i}\ab_j} = \projection{\sigma^2}{\cup_{j\neq i}\ab_j}$.
A possibility that we have not explored is applying \emph{partial-order reduction} to eliminate redundant interleavings in $\Obs_{\cup_{j\neq i}\ab_j}^+$.

Since the language of an LTS is prefix-closed, we can -- in some cases -- extend the function $T$ that is part of the observation table without performing membership queries.
Concretely, if $T(\sigma) = 0$ then we can set $T(\sigma \cdot \sigma') = 0$ for any trace $\sigma'$.
Dually, if $T(\sigma \cdot \sigma') = 1$ then we set $T(\sigma) = 1$.

\subsection{Relaxing the Assumption $\Omega \models \lang(\SUL)$}
\label{sub:gen_adapter}
\cref{alg:comp_i_queries_basic,alg:lstar_backtracking} learn a model of $\lang(\SUL)$ as a parallel composition over a given $\Omega \models \lang(\SUL)$. In \cref{sec:algo}, we reuse these algorithms in the more general setting where only $\Omega \models \Obs$ is guaranteed, with no assurance that future observations remain consistent with $\Omega$. A key invariant in this setting is that membership queries, as constructed by the Adapter, cannot invalidate the current distribution, as the following proposition states.
\begin{prop}
\label{pr:MQnotBreaking}
Consider an observation function \(\Obs\) and a distribution \(\Omega\) such that \(\Omega\models \Obs\). 
Then, for any global membership query result \((\sigma,b) \in \ab^{\star}_{SUL} \times \{\plus,\minus\}\), \(\Omega\models \Obs\cup\{(\sigma,b)\}\). 
\end{prop}
\begin{proof}
First note that \(\sigma\not\in\Dom(\Obs)\), as otherwise $\projection{\sigma}{\ab_i}$ would be in $\Dom(\Obs_{\ab_i})$ and the local query would have not been issued.
Let \(\Learn_i\) the learner that made the query and write \(\Obs'=\Obs\cup\{(\sigma,b)\}\). We apply \Cref{pr:closed_obs_distribution} and verify the condition $\Obs'(\sigma') = \bigwedge_{\ab_j \in \Omega} \Obs'_{\ab_j}(\projection{\sigma'}{\ab_j})$ separately for $\sigma' \in \Dom(\Obs)$ and $\sigma' = \sigma$.

\begin{description}
    \item[$\underline{\sigma' \in \Dom(\Obs)}$] by construction of the Adapter, since $\sigma$ is an interleaving with some $\sigma' \in \Obs^+_{\cup_{j\neq i}\ab_j}$ (line~\ref{line:adap_Pi} of \Cref{alg:comp_i_queries_basic}), we have $\projection{\sigma'}{\ab_j} \in \Obs^+_{\ab_j}$, hence $\Obs_{\ab_j}(\projection{\sigma}{\ab_j}) = \plus$, for all $j \neq i$ (note that some of these projections may be $\epsilon$). Therefore, adding $(\sigma, b)$ does not alter $\Obs'_{\ab_j}$ on $\projection{\Dom(\Obs)}{\ab_j}$; and since $\projection{\sigma}{\ab_i} \notin \Dom(\Obs_{\ab_i})$, the same holds for $j = i$. The condition of \Cref{pr:closed_obs_distribution} thus holds for all $\sigma'$ already in $\Dom(\Obs)$.
    
    \item[$\underline{\sigma' = \sigma}$] if $b = \plus$, then $\Obs'(\sigma) = +$ and, by definition of local observation functions, $\Obs'_{\ab_j}(\projection{\sigma}{\ab_j}) = \plus$ for all $j$, so the condition holds. If $b = \minus$, then $\Obs'(\sigma) = \minus$. For $j \neq i$, $\Obs'_{\ab_j}(\projection{\sigma}{\ab_j}) = \plus$ as argued above. For $j = i$, since $\projection{\sigma}{\ab_i} \notin \Dom(\Obs_{\ab_i})$, no trace in $\Dom(\Obs)$ projects to $\projection{\sigma}{\ab_i}$, so $\Obs'_{\ab_i}(\projection{\sigma}{\ab_i}) = \Obs'(\sigma) = \minus$. Hence $\bigwedge_{\ab_j \in \Omega} \Obs'_{\ab_j}(\projection{\sigma}{\ab_j}) = \minus = \Obs'(\sigma)$.
\end{description}
\end{proof}
Since membership queries preserve the distribution, any inconsistency must arise from an equivalence query; how to detect and resolve such inconsistencies is the subject of the following sections.

\section{Counter-examples to Distributions}
\label{sec:CED}

Recall that our approach starts from the finest distribution $\Omega_{singles} = \{\{a\} \mid a \in \ab_{\SUL}\}$ and gradually updates it until it models the observation function.
Building on the groundwork of~\cref{subsec:distributions} and especially~\cref{pr:closed_obs_distribution}, we define \emph{counter-examples to a distribution}, i.e., witnesses of inconsistency between a distribution and the observations, and characterize when they arise (\cref{sub:CED}), show how to resolve them (\cref{sub:CED}), and leverage these results to construct a new distribution modeling the observation function (\cref{sub:distri_mod}).
The general algorithm integrating these ideas is presented in \cref{sec:algo}.

\subsection{Counter-example to a Distribution}
\label{sub:CED}
By \cref{pr:closed_obs_distribution}, $\Omega \not\models \Obs$ precisely when $\Obs(\sigma) \neq \bigwedge_{\ab_i \in \Omega} \Obs_{\ab_i}(\projection{\sigma}{\ab_i})$, for some $\sigma$.
Since \(\Obs(\sigma)=\plus\) always implies \(\Obs_{\ab_i}(\projection{\sigma}{\ab_i}) = \plus\) by definition of local observations, this requires a globally negative observation \(\Neg\) together with a set of globally positive observations whose local projections match those of $\Neg$, indicating a mismatch between global and local observations. We formalize this as follows.

\begin{defi}[Counter-example to a distribution]
    A counter-example to \(\Omega\models\Obs\) is a pair \((\Neg,\Pos)\in\Dom(\Obs)\times\Dom(\Obs)^\Omega\) with 
    \begin{itemize}
        \item \(\Neg\) a negative observation \(\Obs(\Neg)=-\);
        \item \(\Pos\) a function that maps each \(\ab_i\in\Omega\) to a positive observation \(\sigma_{\ab_i}\), i.e., \(\Obs(\sigma_{\ab_i})=+\), such that \(\projection{\Neg}{\ab_i}=\projection{\sigma_{\ab_i}}{\ab_i}\).
    \end{itemize}
    We call $\supp(P)$ the \emph{positive image} of the counter-example.
    We write \(\CED(\Omega,\Obs)\) for the set of such counter-examples.
    \label{def:distr-cex}
\end{defi}
Although these counter-examples are not necessarily related to learning, 
we use the same terminology as in active learning. This is because the two concepts are directly linked in our case, as will be explained later.
\begin{exa}
    \label{ex:CED}
    Reusing the observation function \(\Obs\) and the singleton distribution \(\Omega_{singles}\) defined in~\Cref{ex:modelObs}, for every element of $\CED(\Obs,\Omega_{singles})$ we have $\Neg=b$. and \(\Pos(\{b\}) = ab\). For the remaining elements of \(\Omega\), there are more choices: 
    $\{a\}$ can be mapped to either $\epsilon$ or $c$; $\{c\}$ to either $\epsilon$, $a$ or $ab$; and \(\{d\}\) to either $\epsilon$, $a$, $ab$ or $c$.
\end{exa}
Proposition~\ref{pr:closed_obs_distribution}, specialized to our definition of counter-examples, yields the following corollary, which will be used in the following to detect that a distribution is a model.

\begin{cor}
    \label{cr:CEtoDist}
    \(\Omega\models\Obs\iff \CED(\Omega,\Obs)=\emptyset\).
\end{cor} 

\begin{proof}
    We prove the two implications separately.

    \begin{description}
        \item[$\implies$] We reason by contrapositive. Suppose \(\CED(\Omega,\Obs)\neq \emptyset\) and let \((\Neg,\Pos)\) be a counter-example. Then \(\Obs(\Neg)=\minus\) and, for all $\ab_i\in\Omega$, $\Obs_{\ab_i}(\projection{\Neg}{\ab_i}) = \Obs_{\ab_i}(\projection{\Pos(\ab_i)}{\ab_i}) = \plus$, by definition of a counter-example and local observation functions. By ~\Cref{pr:closed_obs_distribution}, \(\Omega\not\models\Obs\). 

        \item[$\impliedby$] Suppose $\CED(\Omega,\Obs) = \emptyset$. Then, by definition of counter-example, for any \(\sigma\in\Dom(\Obs)\) with \(\Obs(\sigma)=\minus\) there is some $i$ such that every $\sigma' \in \Dom(\Obs)$ satisfying $\projection{\sigma'}{\ab_i} = \projection{\sigma}{\ab_i}$ also has $\Obs(\sigma') = \minus$. Hence $\Obs_{\ab_i}(\projection{\sigma}{\ab_i}) = \minus$, so $\bigwedge_{1 \leq i \leq n} \Obs_{\ab_i}(\projection{\sigma}{\ab_i}) = \minus$. Since positive observations trivially satisfy $\bigwedge_{1 \leq i \leq n} \Obs_{\ab_i}(\projection{\sigma}{\ab_i}) = \plus$, \cref{pr:closed_obs_distribution} gives $\Omega \models \Obs$. 
    \end{description}

\end{proof}

\subsection{Resolving a Counter-example}
\label{sub:ext_dist}
Given a distribution \(\Omega\) and a fixed observation function \(\Obs\), one key question is how to extend \(\Omega\) to a new distribution \(\Omega'\) modeling \(\Obs\). This is a difficult problem, as new counter-examples can arise when extending a distribution. 
In this subsection, we explain how to resolve a single counter-example as a first step.  

When a counter-example \((\Neg,\Pos)\) to \(\Omega\models\Obs\) exists, it reveals a limitation in the distribution \(\Omega\): the projections of \(\Neg\) coincide with projections of elements in \(\Pos\), making them indistinguishable under the current components.
To resolve such counter-examples, it is thus necessary and sufficient to augment \(\Omega\) with new components that disrupt this matching.
In the following, we will fix $(\Neg,\Pos) \in \CED(\Omega,\Obs)$ as a counter-example to \(\Omega\models\Obs\). 

More precisely, 
for each pair of traces \((\Neg,\sigma)\) with \(\sigma\in\Pos\), 
it suffices to identify a \emph{discrepancy} between them.
There are two types of discrepancies: \emph{multiplicity discrepancies} and \emph{order discrepancies}.
A multiplicity discrepancy is a symbol occurring a different number of times in each trace.
For this, given a trace $\sigma$, let $\multiab(\sigma)$ denote the multiset of symbols occurring in $\sigma$.
Note that $\multiab(\sigma) = \multiab(\sigma')$ if and only if $\sigma$ is a permutation of $\sigma'$.
The symmetric difference of multisets $A$ and $B$ is denoted $A \symdif B$.
\begin{defi}[Multiplicity discrepancy]
	Given a $\ab_i \in \Omega$, the set of \emph{multiplicity discrepancies} for $\ab_i$ is
	\( \DSet_m^{\ab_i}(\Neg, \Pos) = \multiab(\Neg) \symdif \multiab(\Pos(\ab_i)) \).
\end{defi}
We now define an \emph{order discrepancy}, \ie, a pair of symbols whose relative positions differ between the traces.
We do this by considering whether symbols that are \emph{not} a multiplicity discrepancy, \ie, those appearing the same number of times in both traces, are permuted.
We choose the permutation such that the relative order of identical symbols is maintained.
\begin{defi}[Order discrepancy]
    \label{def:disc_ord}
	Given $\ab_i\in\Omega$, let \(\nodiff=\ab\setminus (\multiab(\Neg) \symdif \multiab(\Pos(\ab_i)))\) be the symbols on which $\Neg$ and $\Pos(\ab_i)$ agree and define \(\Neg' = \projection{\Neg}{\nodiff}\).
	Let $\pi$ be the unique permutation such that $\Neg' = \pi(\projection{{\Pos(\ab_i)}}{\nodiff})$ and $\Neg'[j] = \Neg'[k] \implies \pi(j) < \pi(k)$, for all $j < k$.
	The set of \emph{order discrepancies} for $\ab_i$ is then:
	\[ \DSet_o^{\ab_i}(\Neg, \Pos) = \{ \{\Neg'[j],\Neg'[k]\} \mid k < j \land \pi(k) > \pi(j) \} \]
\end{defi}
Multiplicity and order discrepancies can be found in linear and quadratic time, respectively.
Finally, we define the \emph{discrepancies} for a counter-example as sets that contain at least a discrepancy of either type for each alphabet in \(\Omega\).
\begin{defi}[Discrepancy set]
    A set $\Dis \subseteq \ab$ is a \emph{discrepancy} for \((\Neg,\Pos)\) iff for all $\ab_i \in \Omega$, either $\DSet_m^{\ab_i}(\Neg,\Pos) \cap \Dis \neq \emptyset$ or there is $\Dis^{\ab_i} \in \DSet_o^{\ab_i}(\Neg,\Pos)$ such that $\Dis^{\ab_i} \subseteq \Dis$. We write \(\DSet(\Neg,\Pos)\) for the set of all discrepancies for the counter-example \((\Neg,\Pos)\). 
\end{defi} 
For a set of counter-examples $\{ce_1,\dots,ce_n\}$, we write $\DSet(\{ce_1,\dots,ce_n\}) = \{ \{ \Dis_1,\ldots, \Dis_n \} \mid \forall i \ldotp \Dis_i \in \DSet(ce_i) \}$, representing all possible selections of one discrepancy per counter-example.
\begin{exa}[Multiplicity discrepancy]
    \label{ex:dis}
Following~\Cref{ex:CED} with the singleton distribution \(\Omega_{singles}=\{\{a\},\{b\},\allowbreak \{c\},\{d\}\}\), consider the counter-example \((\Neg,\Pos)=(b,(\{a\}\mapsto\epsilon, \{b\}\mapsto ab,\{c\}\mapsto\epsilon,\{d\}\mapsto\epsilon))\). We find the following multiplicity discrepancies: 
$\DSet_m^{\ab_i}(\Neg,\Pos) = \{b\}$, for $\ab_i \in \{ \{a\}, \{c\}, \{d\}\}$, 
 because $b$ occurs once in $\Neg$ vs. zero times in $\Pos(\ab_i)$, 
and \(\DSet_m^{\{b\}}(\Neg,\Pos)=\{a\}\). Hence, \(\DSet(\Neg,\Pos)\) includes all subsets of $\{a,b,c,d\}$ containing \(\{a,b\}\). 
For a different counter-example such as \((\Neg,\Pos')=(b,(\{a\}\mapsto c, \{b\}\mapsto ab,\{c\}\mapsto ab,\{d\}\mapsto c))\), we obtain \(\DSet_m^{\{a\}}(\Neg,\Pos')=\DSet_m^{\{d\}}(\Neg,\Pos')=\{b,c\}\) and \(\DSet_m^{\{b\}}(\Neg,\Pos')=\DSet_m^{\{c\}}(\Neg,\Pos')=\{a\}\), so \(\DSet(\Neg,\Pos')\) includes any subset of $\{a,b,c,d\}$ that contains either \(\{a,b\}\) or \(\{a,c\}\). 
\end{exa}
\begin{exa}[Order discrepancy]
    \label{ex:dis_order}
    Consider a singleton distribution \(\Omega_{singles}=\{\{a\},\{b\}, \{c\}\}\) with
    \(\Obs(abc)=\plus\) and \(\Obs(bac)=\minus\). 
This yields the counter-example \((\Neg,\Pos) = (bac,(\{a\}\mapsto abc,\{b\}\mapsto abc,\{c\}\mapsto abc))\). For each $\ab_i \in \Omega_{singles}$, we find $\DSet_o^{\ab_i}(\Neg,\Pos) = \{\{a,b\}\}$. Intuitively, these discrepancies reveals that singleton components allow for all permutations of $a$ and $b$, but the observation function forbids some of them. Therefore, \(\DSet(\Neg,\Pos)\) includes all subsets of $\{a,b,c\}$ containing \(\{a,b\}\). 
\end{exa}
We can now state that discrepancies are both sufficient and necessary additions to a distribution in order to eliminate their counter-examples. 

\begin{prop}
    \label{pr:fixingCED}
    Suppose there exists \((\Neg,\Pos)\in\CED(\Omega,\Obs)\). 
    For each discrepancy \(\Dis\in\DSet(\Neg,\Pos)\), 
    \((\Neg,\Pos)\not\in\CED(\Omega\cup \{\Dis\},\Obs)\).
    Conversely, for any distribution \(\Omega'\) of \(\ab\) such that \(\Dis\not\subseteq\ab_i\), for all \(\Dis\in\DSet(\Neg,\Pos)\) and all \(\ab_i\in\Omega'\), \((\Neg,\Pos) \in \CED(\Omega',\Obs)\).
\end{prop}
\begin{proof}
    Consider a discrepancy \(\Dis\).
    To show that $(\Neg, \Pos) \notin \CED(\Omega \cup \{\Dis\}, \Obs)$, it suffices to show that no element of $\supp(\Pos)$ can be assigned to $\Dis$ to extend $(\Neg, \Pos)$ to a counter-example in $\CED(\Omega \cup \{\Dis\}, \Obs)$, \ie, that there is no $\sigma_{\ab_i} \in \supp(\Pos)$ such that $\projection{\Neg}{\Dis} = \projection{\sigma_{\ab_i}}{\Dis}$.
    Fix \(\sigma_{\ab_i}\in\supp(\Pos)\).
    \begin{itemize}
        \item If \(\DSet_m^{\ab_i}(\Neg,\Pos)\cap\Dis\neq\emptyset\) then there is $a\in\Dis$ that occurs in different multiplicities in $\Neg$ and $\sigma_{\ab_i}$.
        Hence, $a$ also occurs in different multiplicities in $\projection{\Neg}{\Dis}$ and $\projection{\sigma_{\ab_i}}{\Dis}$, which implies $\projection{\Neg}{\Dis} \neq \projection{\sigma_{\ab_i}}{\Dis}$.
        \item Otherwise, following the notations of \Cref{def:disc_ord}, let $\Neg' = \pi(\projection{\sigma_{\ab_i}}{\nodiff})$ where $\pi$ preserves the order of equal symbols, \ie, $\Neg'[j] = \Neg'[k]$ and $j < k$ imply $\pi(j) < \pi(k)$.
        This ensures that the $l$-th occurrence of any symbol $a$ in $\Neg'$ corresponds to the $l$-th occurrence of $a$ in $\projection{\sigma_{\ab_i}}{\nodiff}$.
        
        By definition of discrepancy, there exists $\{a,b\} \in \DSet_o^{\ab_i}(\Neg,\Pos)$ with $\{a,b\} \subseteq \Dis$.
        Let $(j,k)$ be the corresponding inversion in $\pi$, \ie, $a = \Neg'[j]$, $b = \Neg'[k]$, $j < k$ and $\pi(j) > \pi(k)$.
        By our assumption on $\pi$, we know that $a \neq b$ and furthermore that in $\projection{\sigma_{\ab_i}}{\nodiff}$ at least one more copy of $b$ precedes this occurrence of $a$ (namely the $b$ at position $\pi(k)$).
        Since $\DSet_m^{\ab_i}(\Neg,\Pos) \cap \Dis = \emptyset$, we have $\Dis \subseteq \nodiff$, so 
        \[
            \projection{\Neg}{\Dis} = \projection{\Neg'}{\Dis} \neq \projection{\projection{\sigma_{\ab_i}}{\nodiff}}{\Dis} = \projection{\sigma_{\ab_i}}{\Dis} \enspace .
        \]
    \end{itemize}

    \medskip

    We now prove the converse. Consider \(\Omega'\) as in the statement and fix \(\ab_j\in\Omega'\). In the following, we show that there is $\sigma\in\supp(\Pos)$ such that $\projection{\Neg}{\ab_j} =\projection{{\sigma}}{\ab_j}$. 
    For this, as there is no discrepancy \(\Dis\in\DSet(\Neg,\Pos)\) contained in \(\ab_j\), by definition of discrepancy set we know that there is at least one $\ab_i\in\Omega$ such that no multiplicity or order discrepancy for $\ab_i$ is a subset of $\ab_j$. 
    In particular, \((\,\multiab(\Neg) \symdif \multiab(\Pos(\ab_i))\,) \cap \ab_j = \emptyset\). Hence \(\projection{{\Pos(\ab_i)}}{\ab_j}\) and $\projection{{\Neg}}{\ab_j}$ are equal up to permutation.
    Let $\pi$ and $\Neg'$ be as in \Cref{def:disc_ord}.
    For any inversion $(j,k)$ in $\pi$, we know that \(\{\Neg'[j],\Neg'[k]\}\not\subseteq \ab_j\), since $\{\Neg'[j],\Neg'[k]\} \in \DSet_o^{\ab_i}(\Neg,\Pos)$.
    Hence the restriction of \(\pi\) to the indices corresponding to symbols in \(\ab_j\) has no inversions, and is therefore the identity on that subsequence, which gives $\projection{\Neg}{\ab_j} =\projection{\Pos(\ab_i)}{\ab_j}$.
\end{proof}
\subsection{Extending a Distribution to Model an Observation Function}
\label{sub:distri_mod}
Using the previous subsection as a basis, we leverage structural properties of distributions to restrict the possible counter-examples that can appear when updating the distribution. Finally, we devise an iterative process that is guaranteed to converge to a distribution modeling \(\Obs\).
\subsubsection{Pre-ordering of Distributions}
Distributions can be preordered by their ``connecting power'', i.e., by the extent to which they connect symbols together as part of the same alphabets.
\begin{defi}[Connectivity preorder]
    Given two distributions \(\Omega\) and \(\Omega'\) of alphabet \(\ab\), we say that \(\Omega\) is \emph{less connecting} than \(\Omega'\) and write \(\Omega\partord\Omega'\) when \(\forall \ab_i \in \Omega \ldotp \exists \ab_j \in \Omega' \ldotp \ab_i \subseteq \ab_j\) 
    (equivalently, \(\Omega'\) is said to be \emph{more connecting} than \(\Omega\)).  The relation is strict, written \(\Omega\strpartord\Omega'\), when \(\Omega'\not\partord\Omega\). The relation \(\partord\) forms a preorder with finite chains. 
\end{defi}
We relate this notion to the sets of counter-examples for a fixed observation function to show that adding connections in a distribution makes the counter-example set progress along a preorder. For this, we first define a notion of inclusion for counter-examples.
\begin{defi}[Counter-example inclusion]
Consider two distributions \(\Omega\) and \(\Omega'\) of \(\ab\), 
\((\Neg,\Pos)\) a counter-example to \(\Omega\models\Obs\) and \((\Neg',\Pos')\) a counter-example to \(\Omega'\models\Obs\).
We write \((\Neg,\Pos)\subseteq(\Neg',\Pos')\) whenever \(\Neg=\Neg'\) and \(\supp(\Pos)\subseteq\supp(\Pos')\).  
The strict inclusion \((\Neg,\Pos)\subset(\Neg',\Pos')\) holds whenever \(\Neg=\Neg'\) and \(\supp(\Pos)\subset\supp(\Pos')\).
\end{defi}
In its simplest form, progress means eliminating counter-examples from the current set of counter-examples $\CED(\Omega, \Obs)$. 
However, a counter-example $ce$ might be replaced by new counter-examples $ce'$ such that $ce \subseteq ce'$, which emerge when new connections are added to the distribution.
Hence, progress 
means that some counter-examples 
are either eliminated or replaced by subsuming ones, 
as depicted in \Cref{fig:ce_partord}. 
\begin{figure}
    \centering
    \begin{tikzpicture}
        \def\dset{1.2cm}
            \node (CE) at (-1.2,-\dset) {$CE =$};
            \node[comp state] (ce1) at (0,-\dset) {$ce_1$};
            \node[comp state] (ce2) at (2,-\dset) {$ce_2$};
            \node[comp state] (ce3) at (4,-\dset) {$ce_3$};
            \node[comp state] (ce4) at (6,-\dset) {$ce_4$};

            \node (CE') at (-1.2,0) {$CE' =$};
            \node[rotate=90] at (-1.35,-.5*\dset) {$\partord$};
            \node[comp state] (ce1') at (0,0) {$ce_1$};
            \node[rotate=90] at (0,-.5*\dset) {$=$};
            \node[comp state] (ce3') at (4,0) {$ce_3'$};
            \node[rotate=90] at (4,-.5*\dset) {$\subseteq$};
            \node[comp state] (ce4') at (5.5,0) {$ce_4'$};
            \node[rotate=90] at (5.5,-.5*\dset) {$\subseteq$};
            \node[comp state] (ce4'') at (6.5,0) {$ce_4''$};
            \node[rotate=90] at (6.5,-.5*\dset) {$\subseteq$};
    \end{tikzpicture}
    \caption{The different relations between the elements of two sets of counter-examples \(CE \partord CE'\).}
    \label{fig:ce_partord}
\end{figure}
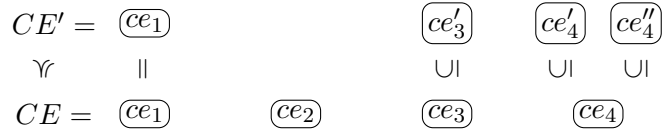
\begin{defi}[Counter-example set preordering]
Consider \(CE\) and \(CE'\) sets of counter-examples. We write 
\(CE\partord CE'\) when 
\(\forall ce'\in CE'.\ \exists ce\in CE.\ ce \subseteq ce' \).
We write $CE \strpartord CE'$ when $CE \partord CE'$ and there exists $ce' \in CE'$ such that either $CE \partord CE' \setminus \{ce'\}$, or every $ce \in CE$ with $ce \subseteq ce'$ satisfies $ce \subset ce'$.
\end{defi}
\begin{exa}
We give a short example of the preorder on set of counter-examples inspired by our running example, that will later appear in~\Cref{ex:sample_run}: 
\[
\{(b,(\{a,b\}\mapsto cb, \{c\}\mapsto\epsilon, \{d\}\mapsto\epsilon))\}
\; \partord \;
\{(b,(\{a,b\}\mapsto cb,\{b,c\}\mapsto ab,\{d\}\mapsto\epsilon))\}\ .
\]
Notice that these two singleton sets' only difference is $\{c\}\mapsto\epsilon$ replaced by $\{b,c\}\mapsto ab$ such that the image of the first is strictly included in the image of the second.
\end{exa}
Using the above definitions, we can prove that increasing the connecting power of a distribution ensures that the set of counter-example progresses. 

\begin{prop}
    \label{pr:orderCED}
    Consider two distributions \(\Omega\) and \(\Omega'\) of \(\ab\).
    We have \(\Omega\partord\Omega' \Rightarrow \CED(\Omega,\Obs)\partord\CED(\Omega',\Obs)\). 
\end{prop}

\begin{proof}
    Suppose $\Omega \partord \Omega'$. If $\CED(\Omega', \Obs) = \emptyset$, the result trivially holds. Otherwise, let $(\Neg, \Pos) \in \CED(\Omega', \Obs)$. For each $\ab_i \in \Omega$, we know that there exists $\ab'_i \in \Omega'$ such that $\ab_i \subseteq \ab'_i$. Since $(\Neg, \Pos) \in \CED(\Omega', \Obs)$, we have $\projection{\Neg}{\ab'_i} = \projection{\Pos(\ab'_i)}{\ab'_i}$, and since $\ab_i \subseteq \ab'_i$, it follows that $\projection{\Neg}{\ab_i} = \projection{\Pos(\ab'_i)}{\ab_i}$. Hence $(\Neg, (\ab_i \mapsto \Pos(\ab'_i))_{\ab_i \in \Omega}) \in \CED(\Omega, \Obs)$, and $(\Neg, (\ab_i \mapsto \Pos(\ab'_i))_{\ab_i \in \Omega}) \subseteq (\Neg, \Pos)$. 
\end{proof}
As an immediate consequence, whenever $\Omega$ has no counter-examples, any distribution $\Omega'$ that is more connecting than $\Omega$ will also have none, i.e., both distributions will model the same observations.
\begin{cor}
\label{cor:empty-more-connecting}
Let $\Omega$ be a distribution of $\ab$ such that $\Omega\models\Obs$.
For any distribution $\Omega'$ of $\ab$ such that $\Omega\partord\Omega'$, we have $\Omega' \models \Obs$.
\end{cor}

\subsubsection{Fixing the distribution}
Using \Cref{pr:fixingCED,pr:orderCED}, from an initial distribution \(\Omega\not\models\Obs\) we can create a more connecting one that entails a strict progression in counter-examples. 

\begin{cor}
    \label{cr:globalProg}
    Suppose that \(\CED(\Omega,\Obs)\neq\emptyset\). 
    For \((\Neg,\Pos)\in\CED(\Omega,\Obs)\), let \(\Dis_{(\Neg,\Pos)}\in\DSet(\Neg,\Pos)\) be a chosen discrepancy. 
    For any non-empty subset $S$ of $\CED(\Omega,\Obs)$, 
    let \(\Omega'=\Omega \cup \{\Dis_{ce}\mid ce \in S\}\).  
    Then \(\Omega\strpartord\Omega'\) and \(\CED(\Omega,\Obs)\strpartord \CED(\Omega',\Obs)\).
\end{cor}
\begin{proof}
    \(\Omega\partord\Omega'\) follows directly from the definition of \(\Omega'\): all elements of \(\Omega\) are preserved. For strictness, note that for each $ce = (\Neg, \Pos) \in S$ and all $\ab_i \in \Omega$, $\Dis_{ce} \not\subseteq \ab_i$. In fact, since $ce \in \CED(\Omega, \Obs)$, we have $\projection{\Neg}{\ab_i} = \projection{\Pos(\ab_i)}{\ab_i}$, but if $\Dis_{ce} \subseteq \ab_i$ this would imply $\projection{\Neg}{\Dis_{ce}} = \projection{\Pos(\ab_i)}{\Dis_{ce}}$, hence $ce \in \CED(\Omega \cup \{\Dis_{ce}\}, \Obs)$, contradicting \cref{pr:fixingCED}. Therefore $\Omega' \not\partord \Omega$ and $\Omega \strpartord \Omega'$.
    
    To show $\CED(\Omega,\Obs) \strpartord \CED(\Omega',\Obs)$, fix $ce \in S$. Since $\Omega \partord \Omega \cup \{\Dis_{ce}\}$, \cref{pr:orderCED} gives $\CED(\Omega, \Obs) \partord \CED(\Omega \cup \{\Dis_{ce}\}, \Obs)$. By \cref{pr:fixingCED}, $ce \notin \CED(\Omega \cup \{\Dis_{ce}\}, \Obs)$, so any $ce' \in \CED(\Omega \cup \{\Dis_{ce}\}, \Obs)$ with $ce \subseteq ce'$ must satisfy $ce \subset ce'$, and hence $\CED(\Omega, \Obs) \strpartord \CED(\Omega \cup \{\Dis_{ce}\}, \Obs)$. Using \cref{pr:orderCED} again, since $\Omega \cup \{\Dis_{ce}\} \partord \Omega'$, we get $\CED(\Omega \cup \{\Dis_{ce}\}, \Obs) \partord \CED(\Omega', \Obs)$. Composing with the strict inequality above gives us our claim.
\end{proof}
\begin{rem}
    \label{rm:optimal_CE}
    \Cref{cr:globalProg} gives us the freedom to select any discrepancy $\Dis_{ce}$, for each counter-example $ce$. We can select discrepancies that result in a least connecting distribution, which yields a locally optimal greedy strategy for progress.
The intuition behind this choice is that it leads to: (1) more components that are individually smaller and easier to learn; and (2) fewer synchronizing actions between components, thus reducing the complexity of coordination among learners.
    \end{rem}
By iteratively applying this Corollary, we can eliminate counter-examples until reaching a distribution that models the observations. This leads to the following convergence result:
\begin{thm}
    \label{th:fixingDist}
    Suppose $\Omega \not\models \Obs$. Iterating the process of \cref{cr:globalProg} converges to a distribution $\Omega' \models \Obs$ with $\Omega \strpartord \Omega'$ in finitely many steps. When $S = \CED(\Omega, \Obs)$ at each step, the number of steps is bounded by $|\Set(\ab)|$.
\end{thm}
\begin{proof}
    By \cref{cr:globalProg}, each step produces a distribution $\Omega''$ with $\Omega \strpartord \Omega''$, so $\Omega \strpartord \Omega'$ holds on convergence.
    We now show convergence.
    Each step eliminates all counter-examples in the chosen non-empty set $S$ (\cref{pr:fixingCED,cr:globalProg}).
    Since $\CED(\Omega, \Obs) \strpartord \CED(\Omega'', \Obs)$, any new counter-example strictly contains one from $\CED(\Omega, \Obs)$.
    In other words, iterating the process makes strict progress: counter-examples are either eliminated or replaced by ones with a strictly larger positive image.
    Since the positive image of any counter-example has at most $|\Omega|$ elements and $|\Omega| \leq |\Set(\ab)|$, we know this growth cannot continue indefinitely.

    When $S=\CED(\Omega,\Obs)$ at each step, all current counter-examples are eliminated simultaneously, so the minimum positive image size among new counter-examples increases by at least one per step, which bounds the number of steps by \(|\Set(\ab)|\).
\end{proof}
\subsubsection{Canonical distributions: inducing a partial order}
Distributions have few constraints: they only need to span the entire alphabet, which leaves room for redundancies.
We propose to remove redundancies without affecting the distribution's connecting power, by removing alphabets completely contained within another.
\begin{defi}[Canonical distribution]
    \label{def:canon_dist}
    Consider a distribution \(\Omega=\{\ab_1,\dots,\ab_n\}\) and 
    \(\textit{Sub}=\{\ab_i\in\Omega\mid \exists \ab_j\in\Omega.\ \ab_i\subset \ab_j\}\). 
    The associated canonical distribution is \(\can{\Omega}= \Omega\setminus \textit{Sub}\). 
\end{defi}
As one would expect, \(\can{\cdot}\) collapses equivalence classes of the preorder \(\partord\) (i.e., \(\Omega\partord\Omega'\) and \(\Omega'\partord\Omega\)) to create a strict partial order.
Canonical distributions allow minimizing the number of alphabets in the distribution while retaining the same connecting power. This means that counter-examples can be easily translated between a distribution and its canonical form, and hence the following proposition.
\begin{prop}
    \label{pr:freesuppress}
    \(\CED(\can{\Omega},\Obs)=\emptyset\Leftrightarrow\CED(\Omega,\Obs)=\emptyset\)
\end{prop}
\begin{proof}
    We proceed by induction on $|\textit{Sub}|$. The base case $\textit{Sub} = \emptyset$ is trivial since $\can{\Omega} = \Omega$. For the inductive step, pick any $\ab_i \in \textit{Sub}$ and let $\Omega' = \Omega \setminus \{\ab_i\}$. Since $\ab_i \in \textit{Sub}$, we have $\can{\Omega'} = \can{\Omega}$. The inductive hypothesis applied to $\Omega'$ thus gives $\CED(\can{\Omega}, \Obs) = \emptyset \Leftrightarrow \CED(\Omega', \Obs) = \emptyset$, so it suffices to show $\CED(\Omega', \Obs) = \emptyset \Leftrightarrow \CED(\Omega, \Obs) = \emptyset$. We show the two implications separately.

    \begin{description}
        \item[$\implies$] We reason by contrapositive. If $(\Neg, \Pos) \in \CED(\Omega, \Obs)$, then for $\Pos' = (\Pos(\ab_j))_{\ab_j \in \Omega'}$ we clearly have \((\Neg,\Pos')\in\CED(\Omega',\Obs)\), since the projection conditions on the remaining components are inherited from $\Pos$.
        \item[$\impliedby$] By definition, $\Omega' \partord \Omega$, so $\CED(\Omega', \Obs) \partord \CED(\Omega, \Obs)$ by \cref{pr:orderCED}. Hence $\CED(\Omega, \Obs) = \emptyset$ implies $\CED(\Omega', \Obs) = \emptyset$.
    \end{description}
\end{proof}

\section{The Orchestrator}
\label{sec:algo}

We are now ready to present the full algorithm. The Orchestrator, outlined in \Cref{subsec:overview}, integrates the local learning framework of \Cref{sec:local_learning} with the distribution-refinement machinery of \Cref{sec:CED}; the high-level workflow is illustrated in \Cref{fig:alg_overview}.
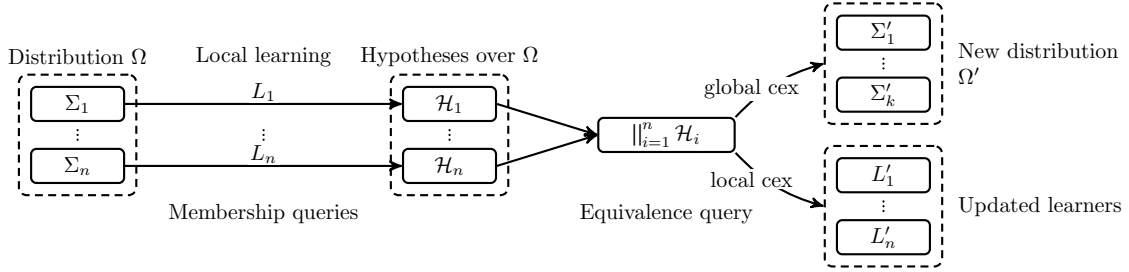
\begin{figure}
    \centering
    \begin{tikzpicture}[
        >=stealth', thick, font=\small, scale=0.82, every node/.style={transform shape},
        lab/.style={midway, fill=white, inner sep=1pt}
    ]

    \def\panelPad{0.15}
    \def\rightLabelGap{0.1cm}
    \def\yMid{0.0}
    \def\yOff{0.50}
    \def\distHypPanelGap{6.0}
    \def\hypToRightPanelGap{7}
    \tikzset{box/.style={draw, rounded corners=2pt, minimum width=1.5cm, minimum height=0.55cm, inner sep=1.2pt, align=center}}
    \tikzset{panel/.style={draw, rounded corners=3pt, densely dashed, inner sep=0pt}}

    \pgfmathsetmacro{\rightCenterOffset}{0.5 + \yOff + \panelPad} %
    \pgfmathsetmacro{\rightPanelX}{\distHypPanelGap + \hypToRightPanelGap} %
    \pgfmathsetmacro{\parX}{(\distHypPanelGap + \rightPanelX)/2} %
    \pgfmathsetmacro{\yTop}{\yMid+\yOff} %
    \pgfmathsetmacro{\yBot}{\yMid-\yOff} %

    \node[box] (dTop) at (0,\yTop) {$\ab_1$};
    \node[box] (dBot) at (0,\yBot) {$\ab_n$};
    \coordinate (dMid) at ($(dTop.south)!0.5!(dBot.north)$);
    \fill ($(dMid)+(0,2.2pt)$) circle (0.5pt);
    \fill (dMid) circle (0.5pt);
    \fill ($(dMid)+(0,-2.2pt)$) circle (0.5pt);
    \node[panel, fit=(dTop)(dBot), inner xsep=\panelPad cm, inner ysep=\panelPad cm] (dPanel) {};
    \node at ($(dPanel.north)+(0,8pt)$) {Distribution $\Omega$};

    \node[box] (h1) at (\distHypPanelGap,\yTop) {$\CH_1$};
    \node[box] (hn) at (\distHypPanelGap,\yBot) {$\CH_n$};
    \coordinate (hMid) at ($(h1.south)!0.5!(hn.north)$);
    \fill ($(hMid)+(0,2.2pt)$) circle (0.5pt);
    \fill (hMid) circle (0.5pt);
    \fill ($(hMid)+(0,-2.2pt)$) circle (0.5pt);
    \node[panel, fit=(h1)(hn), inner xsep=\panelPad cm, inner ysep=\panelPad cm] (hPanel) {};
    \node at ($(hPanel.north)+(0,8pt)$) {Hypotheses over \(\Omega\)};

    \coordinate (llMid) at ($(dPanel.center)!0.5!(hPanel.center)$);
    \fill ($(llMid)+(0,2.2pt)$) circle (0.5pt);
    \fill (llMid) circle (0.5pt);
    \fill ($(llMid)+(0,-2.2pt)$) circle (0.5pt);
    \node at ($(llMid |- dPanel.north)+(0,8pt)$) {Local learning};
    \node at ($(llMid |- dPanel.south)+(0,-8pt)$) {Membership queries};

    \draw[->] (dTop.east) -- node[lab,above] {$L_1$} (h1.west);
    \draw[->] (dBot.east) -- node[lab,above] {$L_n$} (hn.west);

    \node[box] (ndTop) at (\rightPanelX,{\yMid+\rightCenterOffset+\yOff}) {$\ab'_1$};
    \node[box] (ndBot) at (\rightPanelX,{\yMid+\rightCenterOffset-\yOff}) {$\ab'_k$};
    \coordinate (ndMid) at ($(ndTop.south)!0.5!(ndBot.north)$);
    \fill ($(ndMid)+(0,2.2pt)$) circle (0.5pt);
    \fill (ndMid) circle (0.5pt);
    \fill ($(ndMid)+(0,-2.2pt)$) circle (0.5pt);
    \node[panel, fit=(ndTop)(ndBot), inner xsep=\panelPad cm, inner ysep=\panelPad cm] (ndPanel) {};
    \node[anchor=west,align=left] at ($(ndPanel.east)+(\rightLabelGap,0)$) {New distribution \\ $\Omega'$};

    \node[box] (ulTop) at (\rightPanelX,{\yMid-\rightCenterOffset+\yOff}) {$L'_1$};
    \node[box] (ulBot) at (\rightPanelX,{\yMid-\rightCenterOffset-\yOff}) {$L'_n$};
    \coordinate (ulMid) at ($(ulTop.south)!0.5!(ulBot.north)$);
    \fill ($(ulMid)+(0,2.2pt)$) circle (0.5pt);
    \fill (ulMid) circle (0.5pt);
    \fill ($(ulMid)+(0,-2.2pt)$) circle (0.5pt);
    \node[panel, fit=(ulTop)(ulBot), inner xsep=\panelPad cm, inner ysep=\panelPad cm] (ulPanel) {};
    \node[anchor=west,align=left] at ($(ulPanel.east)+(\rightLabelGap,0)$) {Updated learners};

    \node[box, minimum width=2.2cm] (par) at (\parX,\yMid)
        {$\bigparallel_{i=1}^{n}\CH_i$};
    \draw[->] (h1.east) -- (par.west);
    \draw[->] (hn.east) -- (par.west);
    \node at ($(par.center |- dPanel.south)+(0,-8pt)$) {Equivalence query};

    \draw[->] (par.north east) to[bend left=15]
        node[lab,above,pos=0.2] {global cex} (ndPanel.west);
    \draw[->] (par.south east) to[bend right=15]
        node[lab,below,pos=0.2] {local cex}  (ulPanel.west);

    \end{tikzpicture}
    \caption{High-level workflow of the Orchestrator.}
    \label{fig:alg_overview}
\end{figure}

The Orchestrator iteratively performs the following operations. Local learners run in parallel over the current distribution $\Omega$, coordinated by the Adapter, until they all produce hypotheses 
(recall that the current distribution \(\Omega\) can not be disproven by membership queries, as proven in~\ref{pr:MQnotBreaking}); any membership query result is recorded in $\Obs$. The Adapter then composes the local hypotheses and submits the resulting equivalence query to the Teacher. The Orchestrator intercepts the counter-example (if any), records it in $\Obs$, and classifies it as global (if $\Omega \not\models \Obs \cup \{(\sigma_{cex},b)\}$, triggering distribution refinement) or local (forwarded to the learners via the Adapter). The detailed algorithm is given in \Cref{alg:main}.

Crucially, when $(\sigma_{cex},b)$ is global, it generates counter-examples to the distribution (\Cref{def:distr-cex}) whose structure is fully characterized by the following lemma.
\begin{lem}
    \label{lm:CEcorrespondance} 
Given a global counter-example \((\sigma,b)\), let \(\Obs'=\Obs\cup \{(\sigma,b)\}\):
    \begin{itemize}
        \item if \(b=-\), then there is \(\Pos\in\Dom(\Obs)^\Omega\) such that \((\sigma,\Pos)\in\CED(\Omega,\Obs')\). 
        \item else, there is \(\Neg\in\Dom(\Obs)\), $S\subseteq \Omega$ and \(\Pos\in\Dom(\Obs)^{\Omega\setminus S}\) such that 
        \((\Neg,\Pos\cup(\ab_i\mapsto\sigma)_{\ab_i\in S})\in\CED(\Omega,\Obs')\)
    \end{itemize}
    Furthermore, all of the elements of \(\CED(\Omega,\Obs')\) have the above structure.
\end{lem}
\begin{proof}
Since the counter-example is global, we know that \(\Omega\models\Obs\) and \(\Omega\not\models\Obs'\). 
From \Cref{cr:CEtoDist}, we get that \(\CED(\Omega,\Obs)=\emptyset\) and \(\CED(\Omega,\Obs')\neq\emptyset\).

Consider any $(\Neg,\Pos)\in\CED(\Omega,\Obs')$. Since $\CED(\Omega,\Obs)=\emptyset$, $\sigma$ must appear in it: if $b=\minus$ then $\Neg=\sigma$, and if $b=\plus$ then $\sigma\in\supp(\Pos)$. Since $\Dom(\Obs')=\Dom(\Obs)\cup\{\sigma\}$, all remaining entries lie in $\Dom(\Obs)$, yielding the stated structure with $S=\{\ab_i\in\Omega\mid\Pos(\ab_i)=\sigma\}$. As $(\Neg,\Pos)$ was arbitrary, all elements of $\CED(\Omega,\Obs')$ have this form.
\end{proof}
Therefore, based on this lemma, the distribution is augmented with discrepancies for a chosen subset $S$ of distribution counter-examples, following \Cref{cr:globalProg}.
This process eventually converges to provide $\Omega$ such that $\Omega \models \Obs$ (\Cref{th:fixingDist}). The new distribution is then optimized by making it canonical (\Cref{def:canon_dist}) and, if desired, increasing its connectivity. 
The optimization step does not affect counter-example-freeness (by \Cref{pr:freesuppress} and \Cref{cor:empty-more-connecting}) and may be used to reduce synchronizations, which improves performance (see \Cref{sec:exp}).
New learners are then started over the updated alphabets.\footnote{In practice, learners leverage $\Obs$ to partially initialize their observation tables.}
\begin{rem}
    We leave the selection of counter-example set $S$ as an implementation choice. While $S=\CED(\Omega,\Obs)$ maximizes counter-example elimination, finding all counter-examples may be expensive. In our implementation, we process just one counter-example at a time, which in practice often yields a valid distribution after a single update step.
\end{rem}
Our main theorem for this section states that the algorithm terminates and returns a correct model of the SUL.
\begin{algorithm}[t]
	\caption{Orchestrator's algorithm.}
	\label{alg:main}
	
	\SetAlgoHangIndent{4ex}
	\DontPrintSemicolon

    \KwIn{The alphabet of the SUL \(\ab_\mathit{SUL}\) and the Teacher \(\Teach\).}
    \KwInit{\(\Obs=\emptyset\), \(\Omega=\{\{a\}\mid a\in\ab_\mathit{SUL}\}\), \(\Learns=\{\text{new learner }\Learn_i\text{ on }\ab_i\mid \ab_i\in\Omega\}\)}
    \While{True}{
        \ForEach{\(\Learn_i\in\Learns\) in parallel}{
            Locally learn via the Adapter until a hypothesis \(\CH_i\) is returned \label{line:locallearn} \;
            $\Obs \gets \Obs \cup \{(\sigma,b) \mid (\sigma,b) \in \text{global membership queries of } \Learn_i\}$
        }
        \(\CEX \gets\) counter-example returned by $\Equiv(\bigpar_i \CH_i)$\;
        \uIf{$\CEX$ is empty}{
            Return \(\CH_1,\dots,\CH_{|\Omega|}\)\;
        }
		\uElseIf{$\CEX = (\sigma_{cex},b)$}{
			$\Obs \gets \Obs \cup  \{(\sigma_{cex},b)\}$ \;
			\uIf(\tcc*[f]{Global counter-example}){$\Omega \not\models \Obs$}{
				\While{$\CED(\Omega,\Obs) \neq \emptyset$}{
					Pick \(S\subseteq\CED(\Omega,\Obs)\) non-empty \\
					$\Omega \gets \Omega \cup 
                    \Omega'
                    $ for some $\Omega' \in \DSet(S)$
				}
				$Optimize(\Omega)$ \\
				$\Learns \gets
				\{\text{new learner }\Learn_i\text{ on }\ab_i\mid \ab_i\in \Omega\}$ \;
			}
            \Else(\tcc*[f]{Local counter-example}){
                Forward $\sigma_{cex}$ to Adapter
            }
		}
	}
\end{algorithm}
\begin{thm}
    \label{thm:algo}
    The combination of the Orchestrator, the Adapter and the local learner as implemented by \cref{alg:main} terminates and on termination $\lang(\CH_1 \parallel \dots \parallel \CH_{|\Omega|}) = \lang(\SUL)$.
\end{thm}
\begin{proof}
    Correctness on termination is guaranteed by the Teacher: the algorithm returns only when $\Equiv(\CH_1 \parallel \dots \parallel \CH_{|\Omega|})$ replies \emph{yes}, which means $\lang(\CH_1 \parallel \dots \parallel \CH_{|\Omega|}) = \lang(\SUL)$.
    For termination, we show that only finitely many counter-examples of either type can occur.

    \medskip
    \noindent\emph{Global counter-examples.} Each global counter-example triggers lines~11--14, which by \cref{th:fixingDist} produces a distribution $\Omega'$ with $\Omega \strpartord \Omega'$. Every ascending chain in $\strpartord$ is finite and bounded above by $\{\ab_\SUL\}$, so at most finitely many such updates occur.

    \medskip
    \noindent\emph{Local counter-examples.} Fix any interval during which $\Omega = \{\ab_1,\dots,\ab_k\}$ is constant. If only local counter-examples are received, the Adapter behaves exactly as in the fixed-distribution setting, so by the same argument as in the proof of \Cref{thm:fixed_dist} all learners eventually converge to $\lang(\CH_i) = \projection{\lang(\SUL)}{\ab_i}$. At that point the Teacher either replies \emph{yes}, terminating the algorithm, or returns a counter-example to $\lang(\CH_1 \parallel \dots \parallel \CH_k) = \bigpar_i(\projection{\lang(\SUL)}{\ab_i},\ab_i)$. Since $\lang(\SUL) \subseteq \bigpar_i(\projection{\lang(\SUL)}{\ab_i},\ab_i)$ always holds, any such counter-example must be some $(\sigma,\minus)$ with $\sigma \notin \lang(\SUL)$ and $\projection{\sigma}{\ab_i} \in \projection{\lang(\SUL)}{\ab_i}$ for all $i$. By \Cref{lm:CEcorrespondance}, this is a global counter-example, triggering a distribution update and ending the interval. Hence the interval must terminate in finite time: either the Teacher replies \emph{yes}, or a global counter-example is eventually received.

\end{proof}

\begin{rem}
We make no claims regarding the number of components returned by the algorithm. The final distribution may vary depending on counter-example and discrepancy choices. Moreover, as discussed in~\cref{rem:matching_local_languages} different sets of component LTSs can result in the same parallel composition.
\end{rem}
\begin{exa}[Example run]
    \label{ex:sample_run}
    We give an example run where the target SUL is the model of \Cref{fig:running}. For the sake of simplicity, we focus on the global counter-examples and the subsequent distribution updates, considering only one distribution counter-example per step. Moreover, we consistently select a smallest discrepancy for each counter-example as our greedy strategy to minimize the connectivity of the resulting distribution.
    
    We start from \(\Omega_{singles}=\{\{a\},\{b\},\{c\},\{d\}\}\). The local alphabets initially contain only one symbol, so local learners will make membership queries about traces containing exclusively that symbol. This leads to the components depicted below. 
    \begin{center}
        \begin{tikzpicture}[->,>=stealth',auto,semithick, initial text={}, initial distance={10pt}, every state/.style={inner sep=1pt, minimum size=0pt}]
            \def\d{1.4}
    
            \begin{scope}[xshift=-4cm]
                \node                  at (-1.2,0) {$\lts_{\{a\}} = $};
                \node[state,initial] (s0) at (0,0)    {$s_0$};
                \node[state]      (s1) at (\d,0)   {$s_1$};
    
                \path
                (s0) edge node  {$a$} (s1);
            \end{scope}

            \begin{scope}[xshift=0.5cm]
                \node                  at (-1.2,0) {$\lts_{\{b\}} = $};
                \node[state,initial] (s0) at (0,0)    {$s_0$};
            \end{scope}

            \begin{scope}[xshift=3.5cm]
                \node                  at (-1.2,0) {$\lts_{\{c\}} = $};
                \node[state,initial] (s0) at (0,0)    {$s_0$};
    
                \path
                (s0) edge[loop right] node  {$c$} (s0);
            \end{scope}

            \begin{scope}[xshift=7cm]
                \node                  at (-1.2,0) {$\lts_{\{d\}} = $};
                \node[state,initial] (s0) at (0,0)    {$s_0$};
            \end{scope}
        \end{tikzpicture}
    \end{center}
    The first global counter-example is $(ab,\plus)$, yielding several counter-examples to \(\Omega_{singles}\models\Obs\), of which we consider 
    \((b,(\{a\}\mapsto\epsilon, \{b\}\mapsto ab,\{c\}\mapsto\epsilon,\{d\}\mapsto\epsilon))\). The smallest discrepancy for this counter-example is \(\{a,b\}\).
    We use it to update the distribution and obtain \(\Omega_{ab}=\{\{a,b\},\{c\},\{d\}\}\),\footnote{We made the distribution canonical (\Cref{def:canon_dist}) and removed the \(\{a\}\) and \(\{b\}\) components.}, which models the current observations. The new component over $\{a,b\}$ is then learned locally, producing (the $\{c\}$ and $\{d\}$ components are unchanged):
    \begin{center}
        \begin{tikzpicture}[->,>=stealth',auto,semithick, initial text={}, initial distance={10pt}, every state/.style={inner sep=1pt, minimum size=0pt}]
            \def\d{1.4}
    
            \begin{scope}[xshift=-3cm]
                \node                  at (-1.4,0) {$\lts_{\{a,b\}} = $};
                \node[state,initial] (s0) at (0,0)    {$s_0$};
                \node[state]      (s1) at (\d,0)   {$s_1$};
                \node[state]      (s2) at (2*\d,0)   {$s_2$};
    
                \path
                (s0) edge node  {$a$} (s1)
                (s1) edge node {$b$} (s2);
            \end{scope}

            \begin{scope}[xshift=4cm]
                \node                  at (-1.2,0) {$\lts_{\{c\}} = $};
                \node[state,initial] (s0) at (0,0)    {$s_0$};
    
                \path
                (s0) edge[loop right] node  {$c$} (s0);
            \end{scope}

            \begin{scope}[xshift=7.5cm]
                \node                  at (-1.2,0) {$\lts_{\{d\}} = $};
                \node[state,initial] (s0) at (0,0)    {$s_0$};
            \end{scope}
        \end{tikzpicture}
    \end{center}
    The next global counter-example $(cb,+)$, leads to distribution counter-example 
    \((b,(\{a,b\}\mapsto cb, \{c\}\mapsto\epsilon, \{d\}\mapsto\epsilon))\). Its smallest discrepancy is $\{b,c\}$ and the new distribution is \(\Omega_{ab,bc} = \{\{a,b\},\{b,c\},\{d\}\}\). 
    Although the counter-example has been handled, \(\Omega_{ab,bc}\) does not model the observations, as \(\CED(\Omega_{ab,bc},\Obs)\) contains 
    \((b,(\{a,b\}\mapsto cb,\{b,c\}\mapsto ab,\{d\}\mapsto\epsilon))\). Its smallest discrepancy \(\{a,b,c\}\) gives \(\Omega_{abc}=\{\{a,b,c\},\{d\}\}\), modeling the observations.

    To finish our example, the next global counter-example is \((abd,\plus)\). The corresponding distribution counter-example is 
    \((d,(\{a,b,c\}\mapsto \epsilon, \{d\}\mapsto abd))\). There are two smallest discrepancies for this counter-example: \(\{a,d\}\) and \(\{b,d\}\). Selecting \(\{b,d\}\) leads to \(\{\{a,b,c\},\{b,d\}\}\), which models the target language and exactly corresponds to the decomposition of~\Cref{fig:running}. Selecting \(\{a,d\}\) creates unnecessary connections, resulting (after some omitted steps) in either \(\{\{a,b,c\},\{a,d\},\{b,d\}\}\) or \(\{\{a,b,c\},\{a,b,d\}\}\) as a final distribution. 

    Our current implementation selects either discrepancy as both are locally optimal. Finding efficient ways to explore multiple discrepancy choices for globally optimal distributions remains an open challenge for future work.
\end{exa} 
\section{Experiments}
\label{sec:exp}
We initially implemented our approach for learning components given a distribution in a tool called \coal~\cite{NeeleS23a}, which we later extended into \coala (for COmpositional Automata Learner with Alphabet refinement), by adding the ability to refine the alphabets based on global counter-examples.
The tool is implemented in Java and builds on top of LearnLib 0.18.0~\cite{IsbernerHS15}, a library for automata learning.
This allows us to re-use standard data structures, such as observation tables, and compare our framework to a state-of-the-art implementation of \Lstar.
There are a few interesting aspects to the implementation.

First, the local learners are required to construct a minimal LTS that is consistent with an observation table containing wildcards.
Since this is an NP-complete problem~\cite{Gold78}, we resort to the following approach.
We first greedily compute a set of pairwise distinct rows in the observation table, the larger the set, the better.
The minimal LTS necessarily has at least this number of states, so if we find an LTS with exactly this number of states, we are done.
To achieve this, we try the following methods in order: (i) check if a previously computed LTS fits, (ii) the blue-fringe variant of RPNI~\cite{DelaHiguera10} (as implemented in LearnLib), or (iii) a SAT translation using the Z3 solver~\cite{DeMouraB08}.
Of these, only the SAT translation is guaranteed to yield a minimal LTS, but it may also require significant time to run.

Second, as discussed in \Cref{sec:algo}, the theory allows optimizing the distribution.
To achieve this, we first build the hypergraph $(\ab,\Delta)$, where $\Delta$ is the set of all discrepancies found thus far.
In this graph, we construct a distribution $\Omega = \{\ab_1,\dots,\ab_n\}$ that is an \emph{edge cover}: for every edge $\delta \in \Delta$ there is an $i$ such that $\delta \subseteq \ab_i$.
To maximize performance of our algorithms, we have two conflicting goals: minimize the number of synchronizing actions (to help the local learners) and maximize the number of alphabets in the distribution $\Omega$ (to increase the compositional nature of our learning).
With this in mind, we first search for cliques in the graph (i.e. sets of actions that interact strongly), later extending them to ensure the edge cover property.
Finally, we sometimes also merge components to convert synchronizations into local actions.
The related problem of finding in a graph the minimum number of cliques whose union includes all edges is NP-complete~\cite{Orlin77}, suggesting that also here we are dealing with a computationally difficult problem.

\subsection{Evaluation}
We evaluate the effectiveness of our approach in terms of savings in membership and equivalence queries. 
When considering only the execution time on the side of the learner, we do not expect any savings over the mature tools due to the possibly expensive SAT procedure for generating local hypotheses.
However, in most realistic settings, the execution of membership and equivalence queries by the Teacher completely dominates the runtime required.
Hence, any savings in the number of queries boosts practical applicability of automata learning.

To validate our approach, we experiment with learning LTSs obtained from three sources:
\begin{enumerate}
\item 300 randomly generated LTSs varying in structure and size. The details are explained below.
\item 328 LTSs obtained from Petri nets that are provided by the Hippo tool~\cite{WisniewskiBWP23} website;\footnote{\url{https://hippo.uz.zgora.pl/}}
these are often more sequential in nature than our other models.
\item Two scalable realistic models, namely \emph{CloudOpsManagement} from the 2019 Model Checking Contest~\cite{AmparoreO19} and a \emph{producers/consumers} model~\cite[Fig.~8]{Zuberek99}.
\end{enumerate}
The random systems are obtained by computing the parallel composition of a number of randomly generated component LTSs.
This yields an accurate reflection of actual behavioral transition systems~\cite{GrooteVR16}.
Each component LTS has a random number of states between 5 and 9 (inclusive, uniformly distributed) and a maximum number of outgoing edges per state between 2 and 4 (inclusive, uniformly distributed).

We assign alphabets to the components LTSs in five different ways that reflect real-world communication structures, see Figure~\ref{fig:comm_structure}.
Here, each edge represents a communication channel that consists of two synchronizing actions; each component LTS furthermore has two local actions.
The hyperedge in \emph{multiparty} indicates multiparty communication: the two synchronizing actions in such a system are shared by all component LTSs.
The graph that represents the \emph{bipartite} communication structure is always complete, and the components are evenly distributed between both sides.
\emph{Random} is slightly different: it contains $2(n-1)$ edges, where $n$ is the number of components, each consisting of one action; we furthermore ensure the random graph is connected.

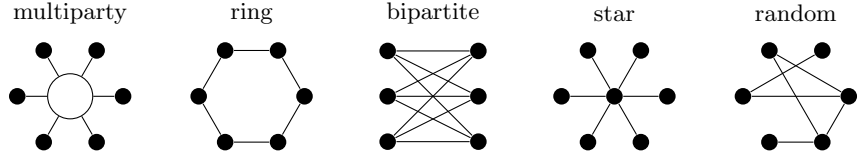
\begin{figure}
	\centering
	\begin{tikzpicture}
		\footnotesize
		\tikzstyle{comp} = [circle,fill=black,inner sep=2.2pt];
		\def\texty{1.1}
		\def\d{0.7}
		\def\x{0.6}
		\def\xsh{2.4cm}
		\begin{scope}
			\node at (0,\texty) {multiparty};
			\node[draw,circle,inner sep=6pt] (6) at (0,0) {};
			\foreach \i in {0,...,5} {
				\node[comp] (\i) at (60*\i:\d) {};
				\path (\i) edge (6);
			}
		\end{scope}
		\begin{scope}[xshift=1*\xsh]
			\node at (0,\texty) {ring};
			\foreach \i in {0,...,5} {
				\node[comp] (\i) at (60*\i:\d) {};
			}
			\path (0) edge (1) (1) edge (2) (2) edge (3) (3) edge (4) (4) edge (5) (5) edge (0);
		\end{scope}
		\begin{scope}[xshift=2*\xsh]
			\node at (0,\texty) {bipartite};
			\foreach \i in {-1,0,1} {
				\foreach \j in {-1,1} {
					\node[comp] (\i\j) at (\j*\x,\i*\x) {};
				}
			}
			\foreach \i in {-1,0,1} {
				\foreach \j in {-1,0,1} {
					\path (\i-1) edge (\j1);
				}
			}
		\end{scope}
		\begin{scope}[xshift=3*\xsh]
			\node at (0,\texty) {star};
			\node[comp] (6) at (0,0) {};
			\foreach \i in {0,...,5} {
				\node[comp] (\i) at (60*\i:\d) {};
				\path (\i) edge (6);
			}
		\end{scope}
		\begin{scope}[xshift=4*\xsh]
			\node at (0,\texty) {random};
			\foreach \i in {0,...,5} {
				\node[comp] (\i) at (60*\i:\d) {};
			}
			\path (0) edge (2) (2) edge (5) (5) edge (4) (0) edge (3) (1) edge (3) (5) edge (0);
		\end{scope}
	\end{tikzpicture}
	\caption{Communication structure of the randomly generated systems. Dots represent components LTSs; edges represent shared synchronizing actions.}
	\label{fig:comm_structure}
\end{figure}

For our five communication structures, we create ten instances for each number of components between 4 and 9; this leads to a total benchmark set of 300 LTSs.
Out of these, 47 have more than 10,000 states, including 12 LTSs of more than 100,000 states.
The largest LTS contains 379,034 states.
\emph{Bipartite} often leads to relatively small LTSs, due to its high number of synchronizing actions.

Using a machine with four Intel Xeon 6136 processors and 3TB of RAM running Ubuntu 20.04, we apply each of three approaches: (i) our black-box compositional approach (\coala), (ii) compositional learning with given alphabets (\coal), and (iii) monolithic learning with \Lstar (as implemented in LearnLib).
\coal can be viewed as an idealized (best-case) baseline where the knowledge of the system decomposition is already available.
Each run has a time-out of 30 minutes.
Since in practice the Teacher dominates runtime, we record the number of
membership and equivalence queries, treating
each query as taking constant time; this eliminates variations in runtime caused by the Teacher.
We also assume the Teacher always returns the shortest counter-example,
which ensures that spurious counter-examples are always eventually corrected
(see \Cref{thm:fixed_dist}); relaxing this assumption is left for future work.
A complete replication package is at~\cite{HenryMNS25Rep}.

\begin{figure}
	\def\w{6.0}
	\centering
	\scalebox{0.9}{
	\begin{tikzpicture}
		\begin{scope}
			\def\dom{3*10^10}
			\def\timeout{1*10^10}
			\begin{loglogaxis}[
				title=Number of membership queries,
				title style={yshift=-0.25cm},
				xlabel=\coal/\coala,
				ylabel=\Lstar,
				xmin=2,xmax=\dom,
				ymin=2,ymax=\dom,
				width=\w cm,height=\w cm,
				xtick pos=left,ytick pos=left,
				ylabel near ticks,
				xlabel near ticks,
				label shift=-.4em,
				]
				\tikzset{every mark/.append style={semithick}}
				\draw (axis cs:2,2) -- (axis cs:\dom,\dom);
				\draw[dashed] (axis cs:2,\timeout) -- (axis cs:\timeout,\timeout);
				\draw[dashed] (axis cs:\timeout,2) -- (axis cs:\timeout,\timeout);
				\addplot[
					scatter,
					scatter src=explicit symbolic,
					only marks,
					gray!40,
					scatter/classes={
						multiparty={mark=x},
						ring={mark=+},
						bipartite={mark=star},
						star={mark=triangle},
						random={mark=o}
					},
					mark=x,
					mark size=1.2pt,
					x filter/.code={\ifthenelse{\equal{#1}{}}{\pgfmathparse{ln(\timeout)}}{}},
					y filter/.code={\ifthenelse{\equal{#1}{}}{\pgfmathparse{ln(\timeout)}}{}},
				] table[x=alphabetsmemQ,y=monomemQ,meta=type,col sep = semicolon] {results-random.csv};
				\addplot[
					scatter,
					scatter src=explicit symbolic,
					only marks,
					mark size=1.5pt,
					scatter/classes={
						multiparty={mark=x,blue},
						ring={mark=+,red},
						bipartite={mark=star,color=green!50!black!70},
						star={mark=triangle,orange},
						random={mark=o,violet!90!white!90}
					},
					x filter/.code={\ifthenelse{\equal{#1}{}}{\pgfmathparse{ln(\timeout)}}{}},
					y filter/.code={\ifthenelse{\equal{#1}{}}{\pgfmathparse{ln(\timeout)}}{}},
				] table[x=compmemQ,y=monomemQ,meta=type,col sep = semicolon] {results-random.csv};
			\end{loglogaxis}
		\end{scope}
		\begin{scope}[xshift=5.9 cm]
			\def\dom{2000}
			\def\timeout{1500}
			\begin{loglogaxis}[
				title=Number of equivalence queries,
				title style={yshift=-0.25cm},
				xlabel=\coal/\coala,
				xmin=1,xmax=\dom,
				ymin=1,ymax=\dom,
				width=\w cm,height=\w cm,
				xtick pos=left,ytick pos=left,
				ylabel near ticks,
				xlabel near ticks,
				label shift=-.4em,
				legend image post style={scale=1.6},
				legend pos=outer north east,
				legend columns=1,
				legend style={font=\small},
				]
				\tikzset{every mark/.append style={semithick}}
				\draw (axis cs:1,1) -- (axis cs:\dom,\dom);
				\draw[dashed] (axis cs:1,\timeout) -- (axis cs:\timeout,\timeout);
				\draw[dashed] (axis cs:\timeout,1) -- (axis cs:\timeout,\timeout);
				\addplot[
					forget plot,scatter,
					scatter src=explicit symbolic,
					only marks,
					gray!40,
					scatter/classes={
						multiparty={mark=x},
						ring={mark=+},
						bipartite={mark=star},
						star={mark=triangle},
						random={mark=o}
					},
					mark=x,
					mark size=1.2pt,
					x filter/.code={\ifthenelse{\equal{#1}{}}{\pgfmathparse{ln(\timeout)}}{}},
					y filter/.code={\ifthenelse{\equal{#1}{}}{\pgfmathparse{ln(\timeout)}}{}},
				] table[x=alphabetsequivQ,y=monoequivQ,meta=type,col sep = semicolon] {results-random.csv};
				\addplot[
					scatter,
					scatter src=explicit symbolic,
					only marks,
					mark size=1.5pt,
					scatter/classes={
						multiparty={mark=x,blue},
						ring={mark=+,red},
						bipartite={mark=star,color=green!50!black!70},
						star={mark=triangle,orange},
						random={mark=o,violet!90!white!90}
					},
					x filter/.code={\ifthenelse{\equal{#1}{}}{\pgfmathparse{ln(\timeout)}}{}},
					y filter/.code={\ifthenelse{\equal{#1}{}}{\pgfmathparse{ln(\timeout)}}{}},
				] table[x=compequivQ,y=monoequivQ,meta=type,col sep = semicolon] {results-random.csv};
				\legend{multiparty,ring,bipartite,star,random};
			\end{loglogaxis}
		\end{scope}
	\end{tikzpicture}}
	\caption{Performance of \Lstar and compositional learning on 300 randomly generated models. Dashed lines indicate time-outs. Results obtained with \coal are in gray, results obtained with \coala are colored.}
	\label{fig:results_random}
\end{figure}
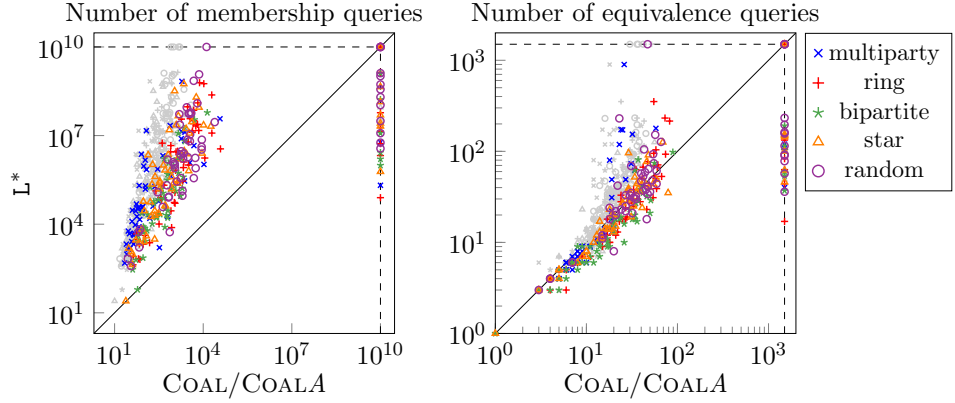

\paragraph*{Random models}
\Cref{fig:results_random} shows the results of the three approaches on our random models.
The colors indicate the various communication structures.
The results obtained with \coal, our best-case baseline, are given in gray, while the results of \coala are colored.
We observe that \coala requires significantly fewer membership queries than \Lstar (note the logarithmic scale) and is closer to the theoretical optimum of \coal; 
the result show 5-6 orders of magnitude of improvement in  a large number of concurrent systems. 
The number of equivalence queries required by \coala is typically slightly higher, but results of larger instances suggest that \coala scales better than its monolithic counterpart.
The data shows that also in the case of equivalence queries, it is not uncommon to gain an order of magnitude of saving by using our approach. 
We note that \coala timeouts occur not due to a high query count, but because of the hypothesis-generating SAT solving routine in local learners.

\paragraph*{Petri Net models}
The results of learning the Hippo models are given in \Cref{fig:results_hippo}.
Here we are not able to run \coal, since the component alphabets are not known. 
\coala does not perform as well as on random models, in particular it requires more equivalence queries than monolithic \Lstar.
This is explained by the fact that these Petri nets contain mostly sequential behavior, \ie, the language roughly has the shape $\lang_1 \cdot a \cdot \lang_2$ for some languages $\lang_1 \subseteq \ab_1^*$ and $\lang_2 \subseteq \ab_2^*$.
Even though our learner is able to find the decomposition $\{ \ab_1 \cup \{a\}, \ab_2 \cup \{a\}\}$, we do not gain much due to the absence of concurrent behavior.
In the Hippo benchmark set, \coala typically finds between two and nine components.
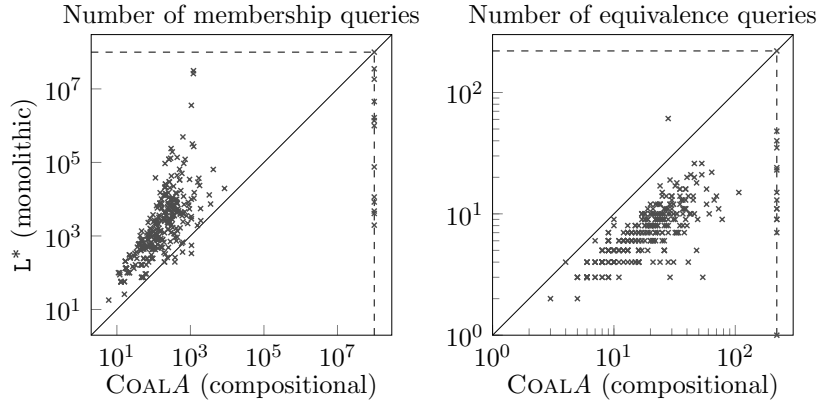
\begin{figure}[htbp]
	\def\w{6.0}
	\centering
	\scalebox{0.9}{
	\begin{tikzpicture}
		\begin{scope}
			\def\dom{3*10^8}
			\def\timeout{1*10^8}
			\begin{loglogaxis}[
				title=Number of membership queries,
				title style={yshift=-0.25cm},
				xlabel=\coala (compositional),
				ylabel=\Lstar (monolithic),
				xmin=2,xmax=\dom,
				ymin=2,ymax=\dom,
				width=\w cm,height=\w cm,
				xtick pos=left,ytick pos=left,
				ylabel near ticks,
				xlabel near ticks,
				label shift=-.4em,
				]
				\tikzset{every mark/.append style={semithick}}
				\draw (axis cs:2,2) -- (axis cs:\dom,\dom);
				\draw[dashed] (axis cs:2,\timeout) -- (axis cs:\timeout,\timeout);
				\draw[dashed] (axis cs:\timeout,2) -- (axis cs:\timeout,\timeout);
				\addplot[
				only marks,
				black!70,
				mark=x,
				mark size=1.5pt,
				x filter/.code={\ifthenelse{\equal{#1}{}}{\pgfmathparse{ln(\timeout)}}{}},
				y filter/.code={\ifthenelse{\equal{#1}{}}{\pgfmathparse{ln(\timeout)}}{}},
				] table[x=compmemQ,y=monomemQ,col sep=semicolon] {results-hippo.csv};
			\end{loglogaxis}
		\end{scope}
		\begin{scope}[xshift=5.9 cm]
			\def\dom{300}
			\def\timeout{220}
			\begin{loglogaxis}[
				title=Number of equivalence queries,
				title style={yshift=-0.25cm},
				xlabel=\coala (compositional),
				xmin=1,xmax=\dom,
				ymin=1,ymax=\dom,
				width=\w cm,height=\w cm,
				xtick pos=left,ytick pos=left,
				ylabel near ticks,
				xlabel near ticks,
				label shift=-.4em,
				]
				\tikzset{every mark/.append style={semithick}}
				\draw (axis cs:1,1) -- (axis cs:\dom,\dom);
				\draw[dashed] (axis cs:1,\timeout) -- (axis cs:\timeout,\timeout);
				\draw[dashed] (axis cs:\timeout,1) -- (axis cs:\timeout,\timeout);
				\addplot[
				only marks,
				black!70,
				mark=x,
				mark size=1.5pt,
				x filter/.code={\ifthenelse{\equal{#1}{}}{\pgfmathparse{ln(\timeout)}}{}},
				y filter/.code={\ifthenelse{\equal{#1}{}}{\pgfmathparse{ln(\timeout)}}{}},
				] table[x=compequivQ,y=monoequivQ,col sep=semicolon] {results-hippo.csv};
			\end{loglogaxis}
		\end{scope}
	\end{tikzpicture}
}
	\caption{Performance of \Lstar and compositional learning on Hippo models. Dashed lines indicate time-outs or out-of-memory.}
	\label{fig:results_hippo}
\end{figure}
\paragraph*{Realistic models}
Finally, \Cref{tab:results-scalable} shows the results of learning two scalable models with relevant parameters indicated.
To sketch what runtime might look like for realistic systems, we include it in this table.
Note, however, that any realistic implementation of the Teacher would require more time than what we have observed in our setting with an idealised Teacher.
No timeouts are reported in the table because all systems were successfully learned within the time limit.

\coala scales well as the SUL size increases, requiring roughly a constant factor more queries than \coal for both CloudOps and producers/consumers.
We remark that practically all runtime of \coala and \coal is spent in the local learners to build hypotheses using SAT queries.
Improving the implementation of local learners would decrease these times significantly; this is left for future work.
As mentioned above, in any practical scenario, the processing of queries by the Teacher forms the bottleneck and \Lstar would be much slower than \coala.  
\begin{table}[htbp]
	\centering
	\caption{Performance of \coala and \Lstar for realistic composite systems. Reported runtimes are in seconds. The number of refinement iterations is listed under `it.' and the number of components found under `com'.}
	\label{tab:results-scalable}
	\resizebox{\textwidth}{!}{
		\begin{filecontents*}{results-realistic.csv}
id;type;sulStates;compmillis;compmemQ;compequivQ;compnumInputs;compnumIterations;compnumComponents;alphabetsmillis;alphabetsmemQ;alphabetsequivQ;alphabetsnumInputs;alphabetsspuriousCex;alphabetsnumComponents;monomillis;monomemQ;monoequivQ;mononumInputs
CloudOps \scalebox{0.7}{W=1,C=1,N=3};nonrandom;690;225211;7686;106;45734;67;9;1272;880;24;5911;0;5;3289;2740128;88;45150932
CloudOps \scalebox{0.7}{W=1,C=1,N=4};nonrandom;1932;235970;9521;115;56566;74;10;1931;923;26;6150;0;6;29617;22252120;216;438575662
CloudOps \scalebox{0.7}{W=2,C=1,N=3};nonrandom;3858;232868;25968;98;168762;63;8;357508;9812;29;87888;3;5;12694;12574560;99;221100718
CloudOps \scalebox{0.7}{W=2,C=1,N=4};nonrandom;10824;708945;33616;111;223085;72;10;555818;9532;30;84821;3;6;143281;91178900;227;1872354180
ProdCons \scalebox{0.7}{K=3,P=2,C=2};nonrandom;1664;3837;685;26;7498;23;5;1402;301;13;4360;0;5;3047;2141165;43;37824466
ProdCons \scalebox{0.7}{K=5,P=1,C=1};nonrandom;170;1034;451;22;5233;14;3;877;118;11;1910;3;3;400;160126;30;3945876
ProdCons \scalebox{0.7}{K=5,P=2,C=1};nonrandom;662;1592;379;24;3990;17;4;931;158;13;2293;0;4;2792;2523625;91;67575625
ProdCons \scalebox{0.7}{K=5,P=2,C=2};nonrandom;2240;5531;697;28;7750;25;5;2350;282;14;4229;1;5;6402;6984705;93;178769086
ProdCons \scalebox{0.7}{K=5,P=3,C=2};nonrandom;8750;17819;1305;31;15242;30;6;6328;604;15;9672;0;6;72865;60186235;187;1558723488
ProdCons \scalebox{0.7}{K=5,P=3,C=3};nonrandom;30344;73688;2454;34;34738;37;7;32586;1269;17;24765;0;7;321539;222567729;193;5927533860
ProdCons \scalebox{0.7}{K=7,P=2,C=2};nonrandom;2816;5263;715;29;9017;25;5;2196;307;15;5521;0;5;16922;15792997;135;565243508
\end{filecontents*}
\pgfplotstabletypeset[
		col sep=semicolon,
		fixed,
		precision=2,
		1000 sep={\,},
		every head row/.style={
			before row= & & \multicolumn{5}{c}{\coala} & \multicolumn{4}{c}{\coal} & \multicolumn{3}{c}{\Lstar}  \\ \cmidrule(lr){3-7}\cmidrule(lr){8-11}\cmidrule(lr){12-14}
			,after row=\midrule},
		every last row/.style={
			after row=\bottomrule},
		columns={id,sulStates,compmillis,compmemQ,compequivQ,compnumIterations,compnumComponents,alphabetsmillis,alphabetsmemQ,alphabetsequivQ,alphabetsnumComponents,monomillis,monomemQ,monoequivQ},
		column type={r},
		columns/id/.style={string type,column type={l},column name=model,
			postproc cell content/.style={@cell content/.add={\scriptsize}{}}
		},
		columns/sulStates/.style={column name=states},
		columns/compmillis/.style={column name=time,zerofill,
			preproc cell content/.code={\pgfkeys{/pgf/fpu}\pgfmathparse{##1/1000}\pgfkeyslet{/pgfplots/table/@cell content}\pgfmathresult \pgfkeys{/pgf/fpu=false}
			}
		},
		columns/compmemQ/.style={column name=memQ},
		columns/compequivQ/.style={column name=eqQ},
		columns/compnumIterations/.style={column name=it.},
		columns/compnumComponents/.style={column name=com},
		columns/alphabetsmillis/.style={column name=time,zerofill,
			preproc cell content/.code={\pgfkeys{/pgf/fpu}\pgfmathparse{##1/1000}\pgfkeyslet{/pgfplots/table/@cell content}\pgfmathresult \pgfkeys{/pgf/fpu=false}
			}
		},
		columns/alphabetsmemQ/.style={column name=memQ},
		columns/alphabetsequivQ/.style={column name=eqQ},
		columns/alphabetsnumComponents/.style={column name=com},
		columns/monomillis/.style={column name=time,zerofill,
			preproc cell content/.code={\pgfkeys{/pgf/fpu}\pgfmathparse{##1/1000}\pgfkeyslet{/pgfplots/table/@cell content}\pgfmathresult \pgfkeys{/pgf/fpu=false}
			}
		},
		columns/monomemQ/.style={column name=memQ},
		columns/monoequivQ/.style={column name=eqQ},
		]{results-realistic.csv}
	}
\end{table} 

\section{Conclusion}
We presented a novel active learning algorithm that automatically discovers component decompositions using only global observations.
Unlike previous approaches, our technique handles a general synchronization scheme common in automata theory and process calculi, without any prior knowledge of component structure. 
We developed algorithms for learning the components, together with a theory of alphabet distributions and their relation to observations; the latter allows us to identify and resolve inconsistencies between distributions and observations, driving an iterative refinement of the component alphabets until
a decomposition consistent with the system under learning is found.
Our \coala implementation dramatically reduces membership queries and achieves better query scalability than monolithic learning on highly concurrent systems.

In the current work, we implemented the local learners with \Lstar. Our setup is sufficiently general to replace this with a modern automata learning algorithm,
provided it supports wildcards and backtracking. Likewise, the translation of membership queries in the Adapter and the choice of global counter-examples in the Orchestrator can each be developed independently of the remaining components. This separation of concerns opens up several concrete directions for future work.

A natural first step is upgrading the local learners from \Lstar to a more efficient algorithm such as TTT~\cite{IsbernerHS14} or \Lsharp~\cite{VaandragerGRW22}. Separately, the SAT-based hypothesis construction that currently accounts for most timeouts in our implementation warrants further attention. Further inspiration for improving the local learners can be drawn from Kruger et al.~\cite{KrugerJR24} or Moeller et al.~\cite{MoellerWSKF023}. Since the Orchestrator plays a large role in the practical performance of our algorithm, we aim to further improve this aspect; this involves developing a theory of counter-example selection and finding efficient strategies for selecting discrepancies that lead to globally optimal distributions. To improve practical applicability, the assumption that the teacher returns a shortest counter-example has to be eliminated. Finally, we will explore compositional learning of other formalisms, such as register automata~\cite{DierlFHJST24} or timed automata~\cite{Waga23}.

\section*{Acknowledgment}
\noindent This research was partially supported by:
EPSRC Standard Grant \emph{CLeVer} (EP/S028641/1);
NWO grant VI.Veni.232.224;
UKRI Trustworthy Autonomous Systems Node in Verifiability -- EP/V026801/2;
EPSRC project Verified Simulation for Large Quantum Systems (VSL-Q) -- EP/Y005244/1;
EPSRC project Robust and Reliable Quantum Computing (RoaRQ), Investigation 009 Model-based monitoring and calibration of quantum computations (ModeMCQ) -- EP/W032635/1;
ITEA/InnovateUK projects GENIUS (600642) and GreenCode (600643)

\bibliographystyle{alphaurl}
\bibliography{refs}

\end{document}